\documentclass[11pt, a4paper, twocolumn, copyright, goog]{google}

\usepackage[authoryear, sort&compress, round]{natbib}
\uselogo{go} 

\title{On Surprising Effectiveness of Masking Updates in Adaptive Optimizers}

\correspondingauthor{taejong.joo@northwestern.edu, cheolmin@google.com}
\reportnumber{} % Leave blank if n/a

\renewcommand{\today}{2026-02-16}

\usepackage{color, colortbl}
\definecolor{Gray}{gray}{0.9}

\usepackage{amsmath}
\usepackage{amssymb}
\usepackage{mathtools}
\usepackage{amsthm}

\usepackage{enumitem}

\usepackage{amsmath,amsfonts,bm}

\def\eqref#1{equation~\ref{#1}}
\def\1{\bm{1}}

\def\vg{{\bm{g}}}

\def\vs{{\bm{s}}}

\def\vu{{\bm{u}}}
\def\vv{{\bm{v}}}
\def\vw{{\bm{w}}}
\def\vx{{\bm{x}}}

\def\mD{{\bm{D}}}

\def\mH{{\bm{H}}}
\def\mI{{\bm{I}}}

\def\mS{{\bm{S}}}

\def\mU{{\bm{U}}}

\DeclareMathAlphabet{\mathsfit}{\encodingdefault}{\sfdefault}{m}{sl}
\SetMathAlphabet{\mathsfit}{bold}{\encodingdefault}{\sfdefault}{bx}{n}

\def\gF{{\mathcal{F}}}

\def\gM{{\mathcal{M}}}

\def\gO{{\mathcal{O}}}

\def\gR{{\mathcal{R}}}

\newcommand{\E}{\mathbb{E}}

\newcommand{\R}{\mathbb{R}}

\newcommand{\sigmoid}{\operatorname{sigmoid}}

\usepackage{xcolor} % http://www.ctan.org/tex-archive/macros/latex/contrib/xcolor

\usepackage[capitalize,noabbrev]{cleveref}

\theoremstyle{plain}
\newtheorem{theorem}{Theorem}
\newtheorem{proposition}[theorem]{Proposition}
\newtheorem{lemma}[theorem]{Lemma}

\theoremstyle{definition}

\newtheorem{assumption}[theorem]{Assumption}
\theoremstyle{remark}

\usepackage{xcolor}
\usepackage{graphicx}

\author[1,*]{Taejong Joo}
\author[2]{Wenhan Xia}
\author[2]{Cheolmin Kim}
\author[2]{Ming Zhang}
\author[2]{Eugene Ie}

\affil[1]{Northwestern University}
\affil[2]{Google}
\affil[*]{Work done while the author was a Student Researcher at Google.}

\begin{abstract}
Training large language models (LLMs) relies almost exclusively on dense adaptive optimizers with increasingly sophisticated preconditioners.
We challenge this by showing that randomly masking parameter updates can be highly effective, with a masked variant of RMSProp consistently outperforming recent state-of-the-art optimizers.
Our analysis reveals that the random masking induces a curvature-dependent geometric regularization that smooths the optimization trajectory. 
Motivated by this finding, we introduce \textbf{\underline{M}omentum-\underline{a}ligned \underline{g}radient \underline{ma}sking (Magma)}, which modulates the masked updates using momentum-gradient alignment.
Extensive LLM pre-training experiments show that Magma is a simple drop-in replacement for adaptive optimizers with consistent gains and negligible computational overhead.
Notably, for the 1B model size, Magma reduces perplexity by over 19\% and 9\% compared to Adam and Muon, respectively.
\end{abstract}

\begin{document}

\maketitle

\section{Introduction}
The availability of dense gradients from a single backward pass in backpropagation \citep{rumelhart1986learning} enables efficient simultaneous parameter updates. 
This efficiency has made dense adaptive optimizers like Adam \citep{kingma2014adam} the de facto standard for large-scale LLM training. 
In contrast, this reliance on dense gradients creates a structural mismatch with sparse update strategies, such as coordinate descent \citep{nesterov2012efficiency,wright2015coordinate,nutini2017let}. 
Consequently, despite their strong performance on the highly nonsmooth optimization problems common in LLM training, sparse methods are rarely used in this setting.

{\begin{algorithm}[tb]
\caption{\colorbox{blue!30}{\pmb{SkipUpdate}} and \colorbox{red!30}{\pmb{Magma}}.}
\label{alg:magma_pseudocode}
\footnotesize
  \begin{algorithmic}
    \STATE {\bfseries Input:} Parameters $\{ \theta_t^{(b)}\}_{b=1}^{B}$, Stochastic gradients $\{ \vg_t^{(b)}\}_{b=1}^{B}$, Updates from a base optimizer $\{ \Delta_t^{(b)}\}_{b=1}^{B}$, First-moment estimates $\{\mu_{t}^{(b)}\}_{b=1}^{B}$
    \FOR{Each block $b \in [B]$}
        \STATE \colorbox{blue!30}{$s_t^{(b)} = 2$ \quad \pmb{(SkipUpdate)}}
        \STATE \colorbox{red!30}{%
        \parbox{0.65\linewidth}{%
        $\tilde{s}_t^{(b)} =
        \sigmoid(\mathrm{cossim}(\mu_t^{(b)},\vg_t^{(b)})/\tau)$\\
        $s_t^{(b)} =
        0.9 s_{t-1}^{(b)} + 0.1 \tilde{s}_t^{(b)}$
        \quad \pmb{(Magma)}%
        }}
        \STATE $m_{t}^{(b)} \sim \mathrm{Bernoulli}(0.5)$
        \STATE $\theta_{t+1}^{(b)} = \theta_t^{(b)} -  s_t^{(b)} m_{t}^{(b)} \Delta_t^{(b)} $
    \ENDFOR
  \end{algorithmic}
\end{algorithm}
}

\begin{figure}
\vskip -0.1in
\begin{center}
\includegraphics[width=0.85\columnwidth]{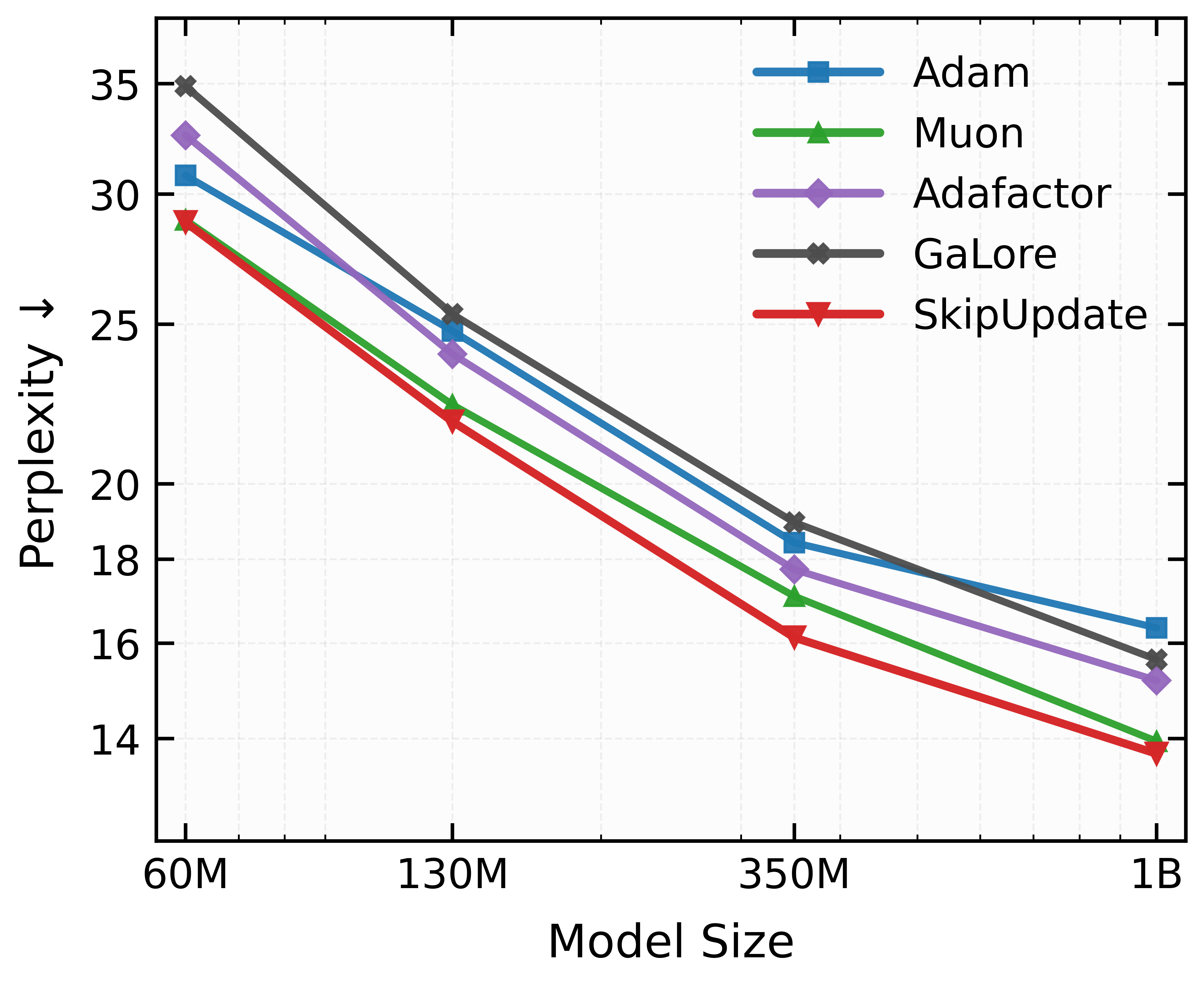}
\end{center}
\vskip -0.15in
\caption{
\textit{
Pre-training performance on C4 across model scales.}
Despite discarding half of updates, SkipUpdate yields substantial improvements over state-of-the-art dense optimizers.
}
\label{fig1:observation_mask}
\vskip -0.2in
\end{figure}

In this work, we challenge this prevailing paradigm with a counter-intuitive empirical finding: we observe that \emph{optimization performance can be substantially improved by randomly masking gradient updates}.
Specifically, we study a variant of RMSProp \citep{hinton2012rmsprop} in which entire parameter blocks are randomly masked at each iteration following a Bernoulli distribution. When a block is masked, its parameter update is skipped for that step; however, moment estimates are still updated densely, and surviving updates are appropriately rescaled to preserve unbiasedness (cf. SkipUpdate in Algorithm \ref{alg:magma_pseudocode}).
From a classical convergence analysis perspective, such random masking would yield a worse worst-case convergence guarantee due to increased stochastic noise in the updates (cf. \S \ref{sec:convergence_analysis}). Further, each parameter effectively receives fewer updates per iteration while the computational cost of gradient computation via backpropagation remains unchanged.
However, despite discarding half of the updates, SkipUpdate consistently outperforms dense optimizers, including the recent state-of-the-art optimizer Muon \citep{jordan2024muon}, which incorporates sophisticated curvature information, across model scales (Figure~\ref{fig1:observation_mask}).

To explain the effectiveness of SkipUpdate, we analyze its geometric regularization effect.
Specifically, we show that block-wise masking induces a curvature-dependent geometric regularization, penalizing updates that align with sharp directions of the loss within each parameter block.
This effect smooths the optimization trajectory and biases the algorithm toward flatter regions of the loss landscape, which are empirically associated with improved generalization in deep networks \citep{hochreiter1997flat,keskar2016large,jastrzkebski2017three}.
Importantly, this regularization emerges implicitly from stochastic noise in the update rule, rather than from explicit curvature computation.

We further investigate whether stochastic masking can be made more effective by exploiting information already present in adaptive optimizers.
We find that modulating the masked updates based on the cosine similarity between the stochastic gradient and the first moment estimate leads to substantial gains.
This mechanism prioritizes momentum-consistent updates by suppressing updates not consistent with the accumulated direction of the gradient.
This yields \textbf{\underline{M}omentum-\underline{a}ligned \underline{g}radient \underline{ma}sking (Magma)} that consistently outperforms both adaptive optimizers and SkipUpdate.
Notably, Magma’s effectiveness increases with model size, consistent with larger models exhibiting more challenging optimization landscapes that require stronger geometric regularization.

Our contributions are:
\underline{\textbf{\textit{(i) Methodological:}}} Magma is a simple optimizer wrapper that improves training stability with improved generalization under the peculiar loss landscape of transformers at no additional computational cost;
\underline{\textit{\textbf{(ii) Theoretical:}}} we show that random gradient masking induces geometric regularization toward flatter trajectories, reducing curvature sharpness and gradient noise;
\underline{\textbf{\textit{(iii) Empirical:}}} we demonstrate consistent gains over state-of-the-art optimizers across diverse pre-training scenarios.

% ========================================================================
% ========================================================================
% ========================================================================
%         Theoretical analysis
% ========================================================================
% ========================================================================
% ========================================================================

\section{Update Masking as a Regularization} \label{subsec:implicit_reg}

\textbf{Notation.}
We denote by $\theta_t$ the model parameters at iteration $t$, which can be partitioned into $B$ disjoint blocks $\{ \theta_t^{(b)} \}_{b=1}^B$. 
Throughout this work, we treat each block as a distinct parameter unit.
We denote by $\nabla_b l(\theta)$ the gradient of the loss $l$ with respect to $\theta^{(b)}$. The stochastic gradient of $l(\cdot)$ at $t$ is denoted by $\vg_t \triangleq \{ \vg_t^{(b)} \}_{b=1}^B$.
$\mathbf{H}_{bb'}(\theta_t)$ denotes the $(b,b')$ block of the Hessian $\nabla^2 l(\theta_t)$. 
We let $(\gF_t)_{t \geq 0}$ be a filtration such that $\theta_t$ is $\gF_t$-measurable and $\E_t[\cdot] \triangleq \E [\cdot | \gF_t]$. 

\textbf{Background: Adaptive optimizers \& moment estimates.}
Adaptive optimizers adjust update magnitudes based on running statistics of past gradients, typically through diagonal or block-diagonal preconditioning. 
Given a learning rate $\eta_t$, these methods update parameters by $\theta_{t+1}^{(b)} = \theta_t^{(b)} - \eta_t \mD_t^{(b)} \vg_t^{(b)},$
where $\mD_t^{(b)}$ is a positive (often diagonal) matrix encoding per-parameter or per-block adaptation. 
Popular adaptive optimizers differ primarily in how $\mD_t^{(b)}$ is constructed or replace $\vg_t^{(b)}$ by a more stable gradient estimator, tracking two running averages of gradient information. The first-moment estimate is a moving average of the gradients $\mu_{t}^{(b)} = \beta_1 \mu_{t-1}^{(b)} + (1 - \beta_1) \vg_t^{(b)}$ and the second-moment estimates track the squared gradients $\vv_t^{(b)} = \beta_2 \vv_{t-1}^{(b)} + (1 - \beta_2) (\vg_t^{(b)})^2,$ where $\beta_1, \beta_2 \in (0,1)$ are constants.
RMSProp uses $\mD_t^{(b)} \approx \text{diag} (\vv_t^{(b)})^{-1/2}$.

\textbf{Main result.}
At iteration $t$, let $\Delta_t = (\Delta_t^{(1)},\ldots,\Delta_t^{(B)})$ denote a block-partitioned update direction from a base optimizer (e.g., $\eta_t \mD_t^{(b)} \vg_t^{(b)}$ for adaptive optimizers). 
SkipUpdate incorporates independent Bernoulli random variables $\{m_{t}^{(b)}\}_{b=1}^B$, where $m_{t}^{(b)} \sim \mathrm{Bernoulli}(p)$ with survival probability $p \in (0,1]$.
Then, stochastic block-wise masking is applied to the updates as 
\begin{equation} \label{eq:rsu_def}
    \tilde{\Delta}_t^{(b)} =  s_t^{(b)} m_{t}^{(b)}\Delta_t^{(b)} , 
\qquad b = 1,\ldots,B,
\end{equation}
yielding $\tilde{\Delta}_t \triangleq (\tilde{\Delta}_t^{(1)},\ldots,\tilde{\Delta}_t^{(B)})$. We set $s_t^{(b)} = 1/ p$ to make the masked update unbiased; that is, $\E_t [ \tilde{\Delta}_t^{(b)} ] = \Delta_t^{(b)}$.

Although random masking preserves the expected update, it fundamentally alters the higher-order behavior of the optimization dynamics.
In the following proposition, which is proved in Appendix \ref{appx:proof_prop1}, we show that SkipUpdate induces a curvature-dependent term in the expected loss decrease.

\begin{proposition}[] \label{lemma:implicit_reg}
Conditioned on $\gF_t$, the expected loss of SkipUpdate (cf. \eqref{eq:rsu_def}) is 
\begin{multline} \label{eq:implicit_reg}
    \mathbb{E}_t\left[l(\theta_t-\tilde{\Delta}_t) \right] = l(\theta_t-\Delta_t) + \\
    \sum_{b=1}^B \frac{1-p}{2p} (\Delta_t^{(b)})^\top  \mathbf{H}_{bb}(\theta_t) \Delta_t^{(b)}
    + \gO(\sum_{b=1}^B \|\Delta_t^{(b)}\|^3).
\end{multline}
\end{proposition}

The term $\gR_t^{(b)} \triangleq \frac{1-p}{2p} (\Delta_t^{(b)})^\top \mathbf{H}_{bb}(\theta_t) \Delta_t^{(b)}$ admits a natural interpretation as a \emph{geometric regularizer}. 
It measures the local curvature of the loss along the update direction $\Delta_t^{(b)}$, weighted roughly by the inverse survival probability. 
Since directions of large positive curvature correspond to sharp increases of the loss, minimizing the expected post-update loss implicitly discourages updates that align with high-curvature directions within each block.

This block-structured regularization is particularly well-motivated in transformers, whose Hessians empirically exhibit pronounced block-diagonal structure \citep{zhang2024transformers,kunstner2024heavy,ormaniec2024does}.
Under this geometry, the dominant curvature interactions occur within blocks. 
Consequently, the block-wise quadratic penalty in Proposition \ref{lemma:implicit_reg} induces a principled second-order regularization that biases optimization toward flatter regions of the loss landscape. This preference for flatter minima has long been related to improved generalization \citep{hochreiter1997flat,keskar2016large} and is targeted by sharpness-aware optimization methods \citep{foret2020sharpness}.

\textbf{Why dense momentum updates matter.}
In Algorithm ~\ref{alg:magma_pseudocode}, momentum states are updated densely even when parameter updates are masked. 
This contrasts with recent subspace optimization methods \citep{pan2024lisa,zhao2024galore,luo2024badam}, in which both parameters and auxiliary states are updated only on selected coordinates.
For example, GaLore \citep{zhao2024galore} selects coordinates via leading gradient singular vectors and updates them over fixed batch intervals.
While this reduces optimizer memory, it repeatedly optimizes a suboptimal coordinate subset, contrary to classical coordinate descent insights \citep{nesterov2012efficiency,nutini2017let}.
Moreover, significant memory savings are not guaranteed in modern LLM training regimes, where  memory consumption is dominated by activations rather than parameters or optimizer states \citep{zhang2024revisiting,shamshoum2025compact}.

Further, SkipUpdate effectively yields a variance-reduced estimator of the true momentum due to the lazy update scheme, resulting in more stable search directions.
Consequently, as shown in Figure \ref{fig:dense-vs-sparse-momentum-update}, the dense momentum updates yield greater stability  with improved generalization, compared to the sparse momentum updates as in the memory-efficient subspace optimizers.

\textbf{Impacts of structured masking.}
Proposition \ref{lemma:implicit_reg} shows that the granularity of masking determines the structure of the induced curvature regularizer, with finer-grained masking progressively eliminating cross-coordinate interactions.
For instance, element-wise masking only penalizes diagonal Hessian entries, yielding a regularization term $\sum_{b=1}^B \sum_{i=1}^{\textrm{dim}(b)} \frac{1-p}{2p} \{ \Delta_t^{(b)} \}_{i}^2   \{ \mathbf{H}_{bb}(\theta_t) \}_{i,i},$ where $\textrm{dim}(b)$ is the dimensionality of the block $b$.
Interestingly, under 130M Llama pre-training on the C4 dataset, SkipUpdate achieves similar perplexities across masking granularities—21.78 (column-wise), 21.73 (element-wise), and 21.81 (block-wise)—all substantially outperforming the RMSProp baseline (22.64).
We conjecture that this near-equivalence reflects the limited ability of diagonal preconditioning to exploit dense within-block curvature, rendering finer-grained masking of marginal practical benefit.
Thus, we adopt block-wise masking for its favorable computational properties, as skipping entire blocks enables efficient operation pruning.

% ========================================================================
% ========================================================================
% ========================================================================
%         Method
% ========================================================================
% ========================================================================
% ========================================================================

\section{Momentum-Aligned Update Masking} \label{sec:mainsection}
SkipUpdate applies a homogeneous masking operation across all blocks.
However, parameters in transformers exhibit substantial heterogeneity, such as markedly different Hessian spectra \citep{zhang2024transformers} and gradient variances \citep{orvieto2025search}. 
Such heterogeneity strongly influences optimization dynamics and therefore motivates a more refined block-adaptive masking strategy, while preserving seamless compatibility with modern adaptive optimizers.

In stochastic optimization, gradient components consistent across iterations tend to carry meaningful optimization signal, whereas rapidly fluctuating components are often dominated by stochastic noise \citep{polyak1964some,riedmiller1993direct,bottou2018optimization}.
Further, recent work interprets moment estimates through the lens of online variational inference \citep{orvieto2025search}. Under this view, the probability of negative alignment satisfies
$\mathbb{P}(\mu^\top g < 0) = \Phi \left(-\frac{\|\mu\|}{\sigma} \right),$ which decays exponentially in the signal-to-noise ratio $\|\mu\|/\sigma$.
Thus, negative alignment events represent statistically abnormal fluctuations, providing a natural signal for identifying destabilizing updates.
This motivates \emph{momentum-gradient alignment} as a principled criterion for modulating update magnitudes.

Based on this insight, we propose Magma, which leverages block-wise momentum-gradient alignment to control the masking process.
Specifically, for each block $b$ at iteration $t$, Magma computes an alignment score $\tilde{s}_t^{(b)} \in (0,1)$ as
\begin{equation} \label{eq:masking_prob}
    \tilde{s}_t^{(b)} = \sigmoid\left( \textrm{cossim}(\mu_t^{(b)}, \vg_t^{(b)}) / \tau \right),
\end{equation}
where $\mu_t^{(b)}$ is the first-moment estimate, $\tau > 0$ is a temperature parameter, and $\textrm{cossim}(\cdot, \cdot)$ is the cosine similarity. 
Cosine similarity is adopted for its scale-invariant property, which is particularly effective in LLM training where gradient norms vary substantially across both parameter blocks and iterations \citep{huang2025spam, wen2025sron}.

Then, Magma modulates the noisy (masked) update based on the alignment score $s_t^{(b)}$, applying the update rule $\tilde{\Delta}_t^{(b)} =  s_t^{(b)} m_{t}^{(b)}\Delta_t^{(b)}, \; b = 1,\ldots,B,$ with $s_t^{(b)} = 0.9 s_{t-1}^{(b)} + 0.1 \tilde{s}_t^{(b)}$ being the exponential moving average.
As a result, Magma encourages coherent optimization trajectories with large $s_t^{(b)}$ values and mitigates oscillatory updates with small $s_t^{(b)}$ values. 
We emphasize that Magma is a drop-in wrapper that multiplies $s_t^{(b)} m_{t}^{(b)}$ into the existing update direction $\Delta_t^{(b)}$ produced by (adaptive) optimizers, introducing no additional memory or computational overhead.
Therefore, practitioners can adopt Magma in existing training pipelines with minimal code changes and no additional resource requirements.

Similar principles underlie Cautious Optimizer \citep{liang2024cautious} and MGUP \citep{chang2025mgup}, which mask or attenuate parameter updates whose stochastic gradients have opposite signs to the first-moment estimate. Similarly, RPROP \citep{riedmiller1993direct} adapts step sizes based on the temporal consistency of gradient signs.
However, unlike Magma, these methods lack structured stochastic masking, and therefore do not induce the curvature-dependent geometric regularization established in \S \ref{subsec:implicit_reg}.

Multiplication by the alignment score (damping) introduces bias into the masked update but substantially improves training stability (cf. Figure \ref{fig:dense-vs-sparse-momentum-update} in Appendix). We tested unbiased alternatives, such as using $\tilde{s}_t^{(b)}$ as the survival probability with $1/\tilde{s}_t^{(b)}$ rescaling, but these consistently resulted in unstable training. Developing a stable yet unbiased masking scheme for Magma remains an important future direction.

% ========================================================================
% ========================================================================
% ========================================================================
%         Experiments
% ========================================================================
% ========================================================================
% ========================================================================
\section{Experiments}

\begin{table}[t]
\caption{Llama 2 pre-training results on the C4 dataset. 
Validation perplexity is reported for four model scales (60M to 1B). $\dagger$ denotes results from \citet{li-etal-2025-taming}; remaining results are ours. RMSProp diverges for the 1B model within the learning rate search space.}
\centering
\resizebox{\columnwidth}{!}{%
\begin{tabular}{lcccc}
\toprule
\textbf{Method} & {\textbf{60M}} & {\textbf{130M}} & {\textbf{350M}} & {\textbf{1B}} \\
\midrule 
Adam & 30.79 & 24.77 & 18.42 & 16.35 \\ 
C-Adam & 29.70 & 23.59 & 18.58 & 15.92 \\
Adam+SGG$\dagger$ & 30.31 & 22.18 & 17.28 & 14.30 \\ 
\rowcolor{Gray} Adam+Magma & \pmb{29.09} & \pmb{22.08} & \pmb{16.41} & \pmb{13.71} \\
\midrule 
LaProp & 29.98 & 23.07 & 18.56 & 16.38 \\
\rowcolor{Gray} LaProp+Magma & \pmb{29.05} & \pmb{22.16} & \pmb{16.37} & \pmb{13.82} \\
\midrule 
Adafactor$\dagger$  & 32.57 & 23.98 & 17.74 & 15.19 \\ 
APOLLO$\dagger$  & 31.55 & 22.94 & 16.85 & 14.20 \\ 
APOLLO+SGG$\dagger$  & 30.18 & 22.52 & 16.54 & 13.95 \\ 
Muon$\dagger$ & 28.93 & 22.34 & 17.09 & 14.52 \\
RMSProp & 29.29 & 22.64 & 17.47 & - \\
\midrule 
\rowcolor{Gray} RMSProp+Magma & \pmb{28.55} & \pmb{21.66} & \pmb{16.16} &  \pmb{13.19} \\ 
\bottomrule
\end{tabular}
}
\label{table:benchmark_c4}
\end{table}

\subsection{Pre-Training Llama}  \label{subsec:pre-training-llms}

We present Llama 2 pre-training results on the C4 dataset \citep{raffel2020exploring} across model sizes of 60M, 130M, 350M, and 1B, following the standardized experimental setup established by \citet{zhao2024galore} (See Appendix \ref{appx:exp_details} for details). For all experiments, we set $\tau = 2$ and apply Magma updates exclusively to the attention and MLP layers. We found that this single configuration performs robustly across a wide range of settings, as supported by the ablation studies in Appendix \ref{appx:ablation_studies}.

Table~\ref{table:benchmark_c4} compares Magma against a wide range of algorithms including Adafactor \citep{shazeer2018adafactor}, Adam, APOLLO \citep{zhu2025apollo}, LaProp \citep{ziyin2020laprop}, and RMSProp, as well as matrix-based optimizers such as Muon and SOAP \citep{vyas2024soap}, and optimization enhancers such as Scaling with Gradient Grouping (SGG) \citep{li-etal-2025-taming} and Cautious Optimizer (C-Adam).

Table~\ref{table:benchmark_c4} demonstrates that Magma consistently improves performance of all base optimizers across model scales. 
When applied to Adam and LaProp, Magma yields significant perplexity reductions compared to their vanilla counterparts and competing enhancers like SGG. Most notably, \textbf{RMSProp+Magma} attains the lowest perplexity across all model sizes (60M--1B), establishing a new state-of-the-art for this benchmark. It outperforms computationally intensive matrix-based optimizers such as Muon and SOAP, as well as sophisticated enhancers like APOLLO+SGG.

Magma proves to be the most effective optimization enhancer among the evaluated methods. When applied to Adam, Magma consistently outperforms alternative enhancement strategies, including Cautious Adam (C-Adam) and Scaling with Gradient Grouping (Adam+SGG). For instance, at the 1B parameter scale, Adam+Magma achieves a validation perplexity of 13.81, significantly surpassing Adam+SGG (14.30) and C-Adam (15.92). This indicates that the masking mechanism in Magma provides a more robust regularization signal than the gradient modification techniques employed by competing enhancers.

Moreover, Magma exhibits favorable scaling: its performance relative to base optimizers increases with model size. As larger models exhibit increasingly irregular and nonsmooth loss landscapes, this trend supports the interpretation of random masking as an effective form of geometric regularization.

\subsection{Pre-Training Nano MoE}
In accordance with the widespread adoption of sparse mixture-of-experts (MoE) architecture in modern LLMs \citep{shazeer2017outrageously,lepikhin2020gshard,fedus2022switch,du2022glam}, we validate Magma in the Nano MoE framework \citep{wolfe2024nanomoe}.
This benchmark involves pre-training an MoE transformer on the OpenWebText dataset \citep{pile}, and we follow the standard configuration and training protocol (see Appendix \ref{appx:nano_moe-exp_details} for details).

MoE models are known to induce significantly more complex and nonsmooth optimization due to dynamic load balancing, sparse token-to-expert routing, and non-uniform gradient flow across parameters \citep{shazeer2017outrageously, fedus2022switch, zoph2022st}. 
Therefore, the MoE architectures serve as a particularly stringent testbed for evaluating robustness of optimization algorithms.

\begin{figure}[t]
% \vskip -0.1in
\begin{center}
\includegraphics[width=0.49\textwidth]{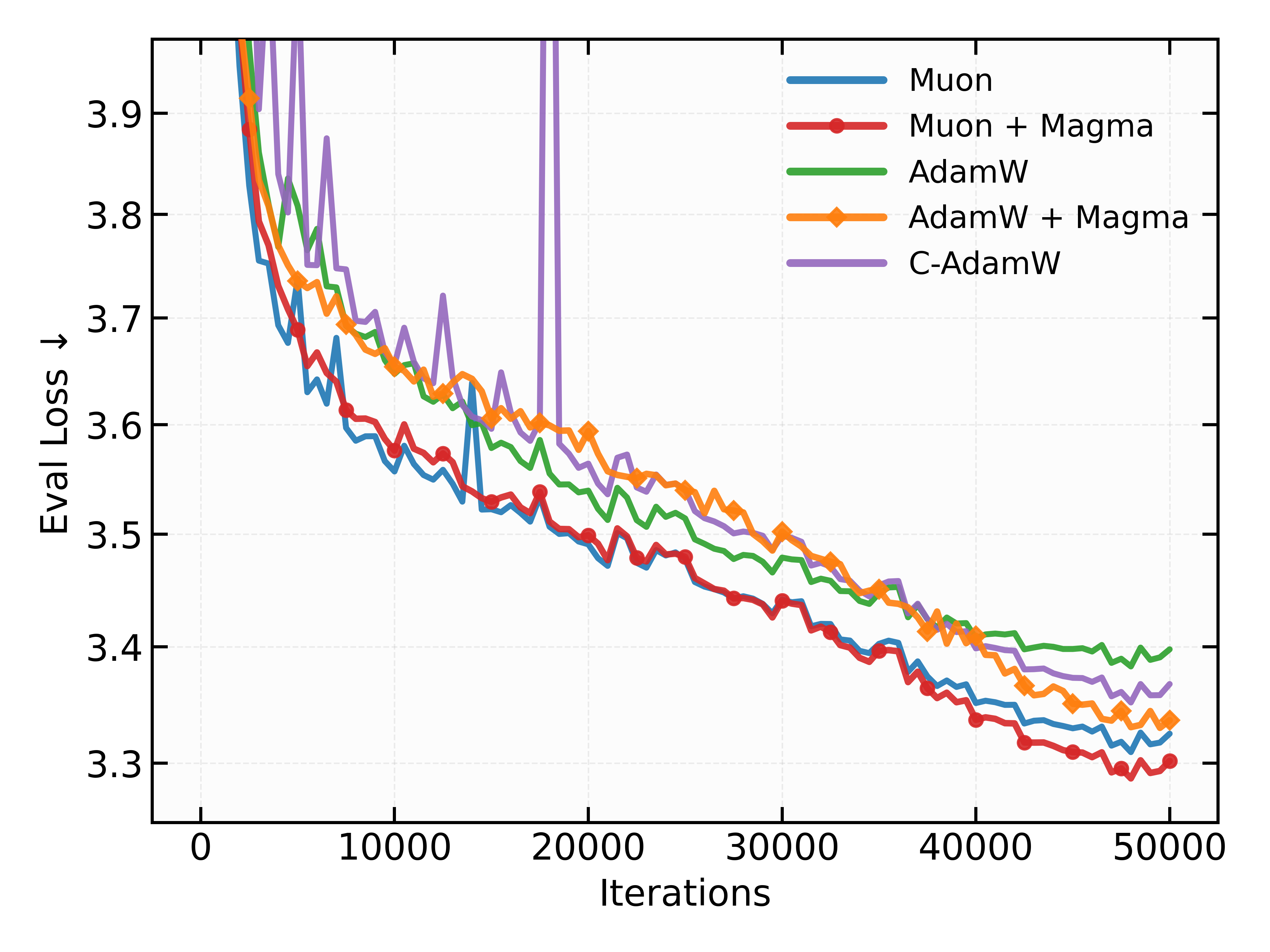}
\end{center}
\vskip -0.15in
\caption{Optimization trajectories of pre-training the Nano MoE model on OpenWebText.} 
\label{fig:nano-moe}
\vskip -0.15in
\end{figure}

Our results in Figure \ref{fig:nano-moe} show that Magma consistently improves performance of both Adam and Muon in the MoE setting. When applied to Adam, Magma exhibits slower convergence during intermediate training but ultimately achieves superior final performance. Consistent with Llama pre-training, Magma also significantly outperforms the Cautious Optimizer (C-Adam), which similarly leverages momentum–gradient alignment, corroborating the effectiveness of Magma’s geometric regularization through structured random masking (\S \ref{sec:mainsection}).

Notably, when combined with Muon, Magma attains the best overall performance, substantially outperforming all baselines. 
This suggests that stochastic masking-based update modulation in Magma and structured preconditioning operate on largely orthogonal aspects of the optimization process, shaping both training dynamics and convergence geometry in complementary ways. Given the increasing reliance on sophisticated preconditioners such as Muon \citep{jordan2024muon} and SOAP \citep{vyas2024soap} in modern MoE-based LLM training, these results highlight Magma as a robust and effective enhancement that can be seamlessly integrated into large-scale optimization pipelines.

\begin{figure}
% \vskip -0.1in
\begin{center}
\includegraphics[width=0.235\textwidth]{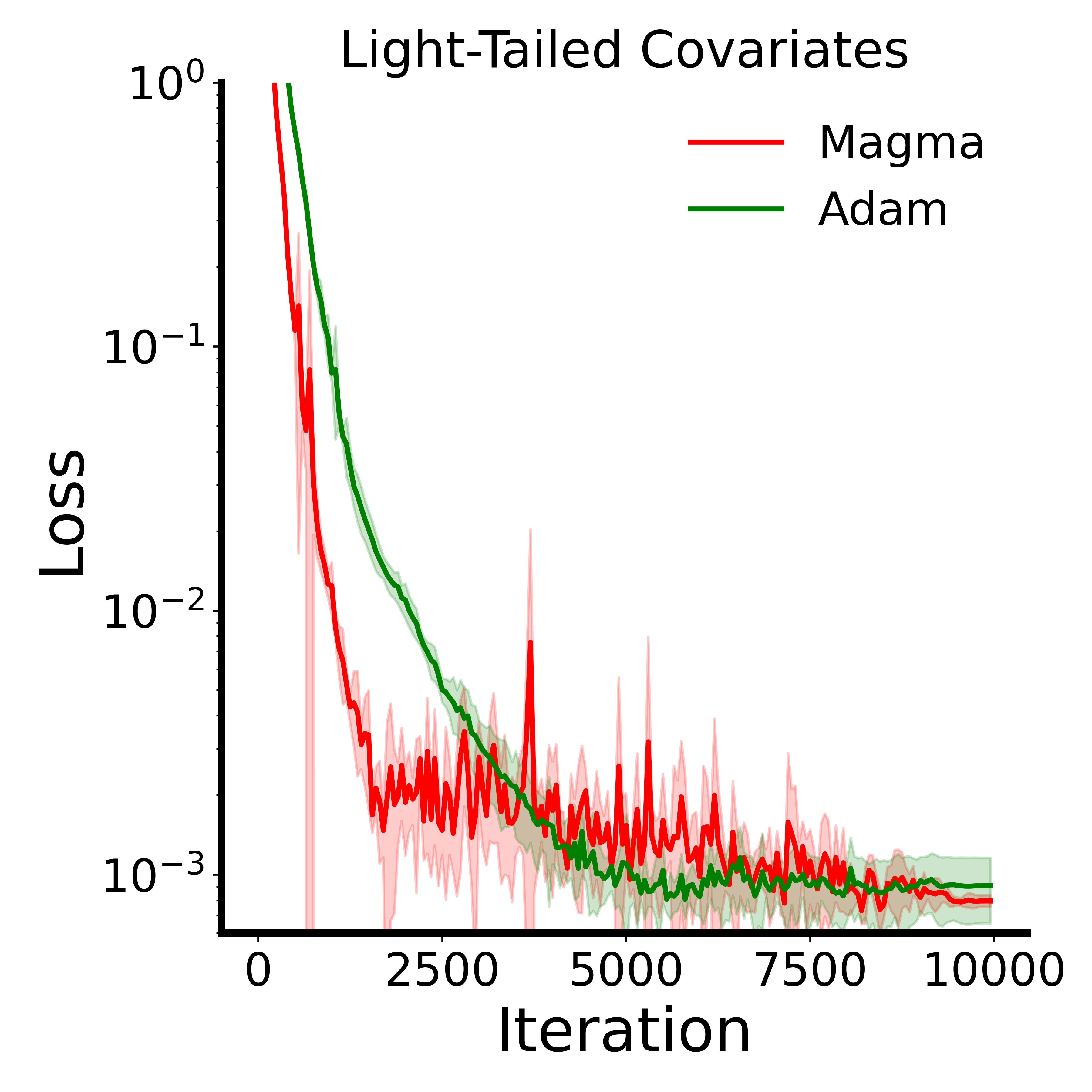}
\includegraphics[width=0.235\textwidth]{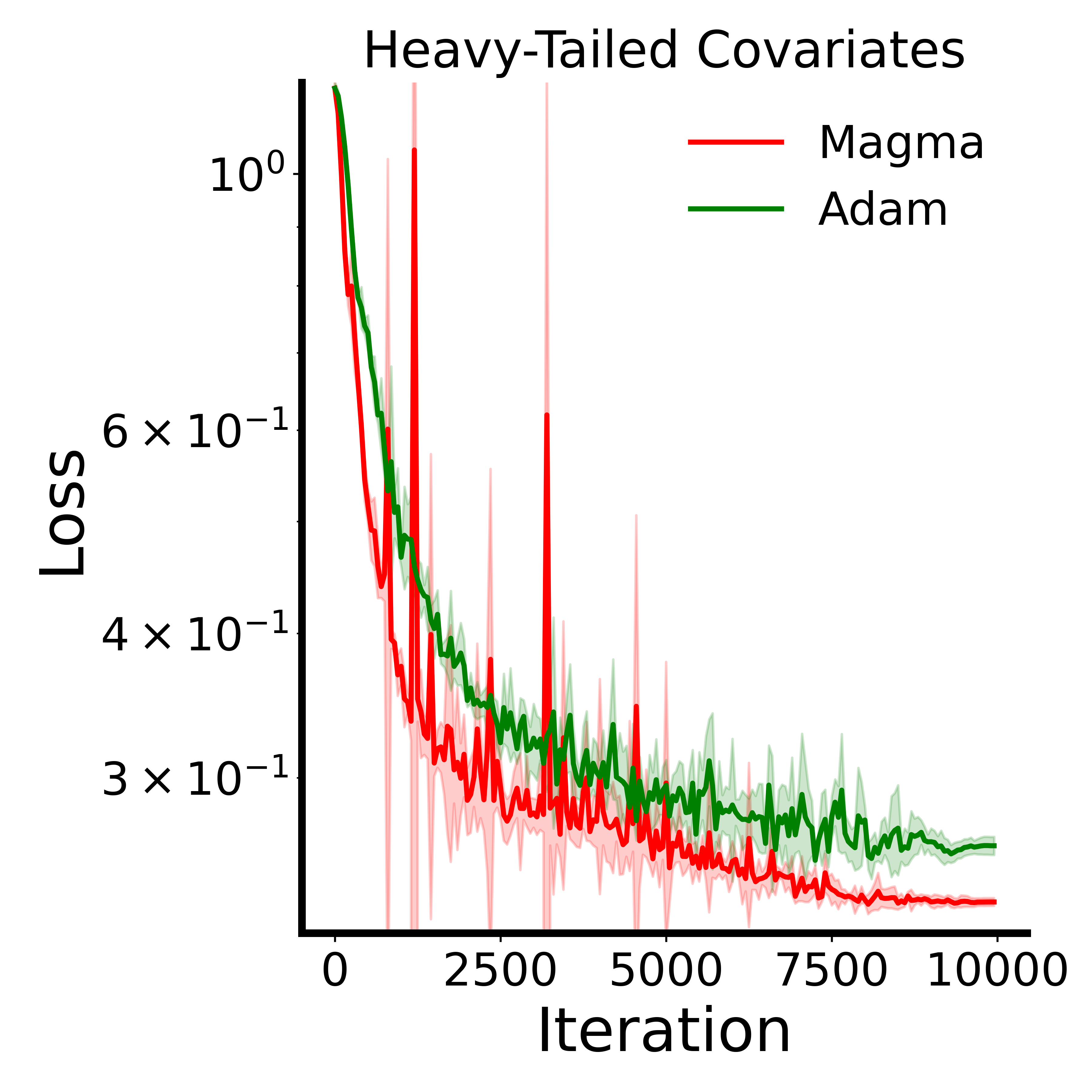}
\includegraphics[width=0.235\textwidth]{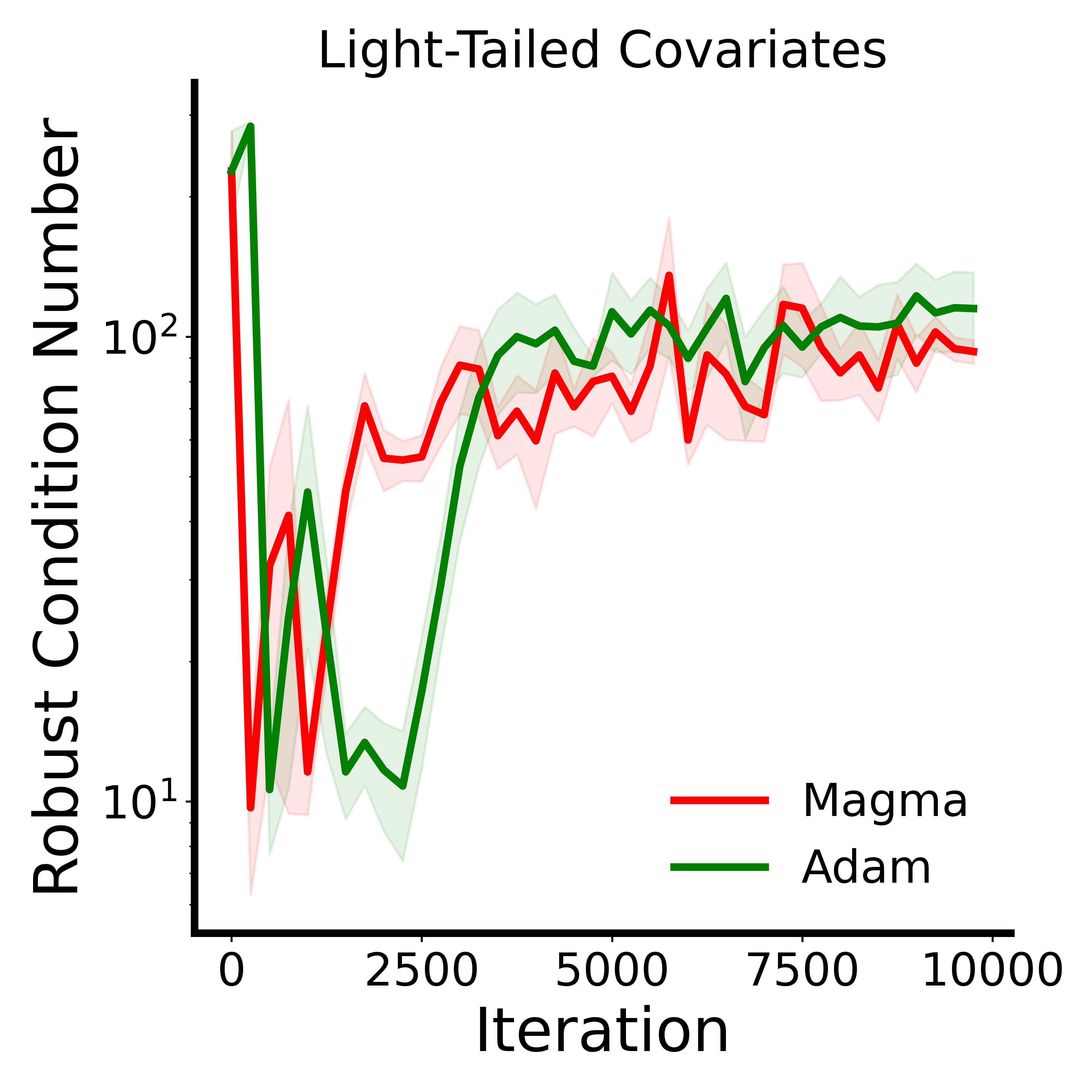}
\includegraphics[width=0.235\textwidth]{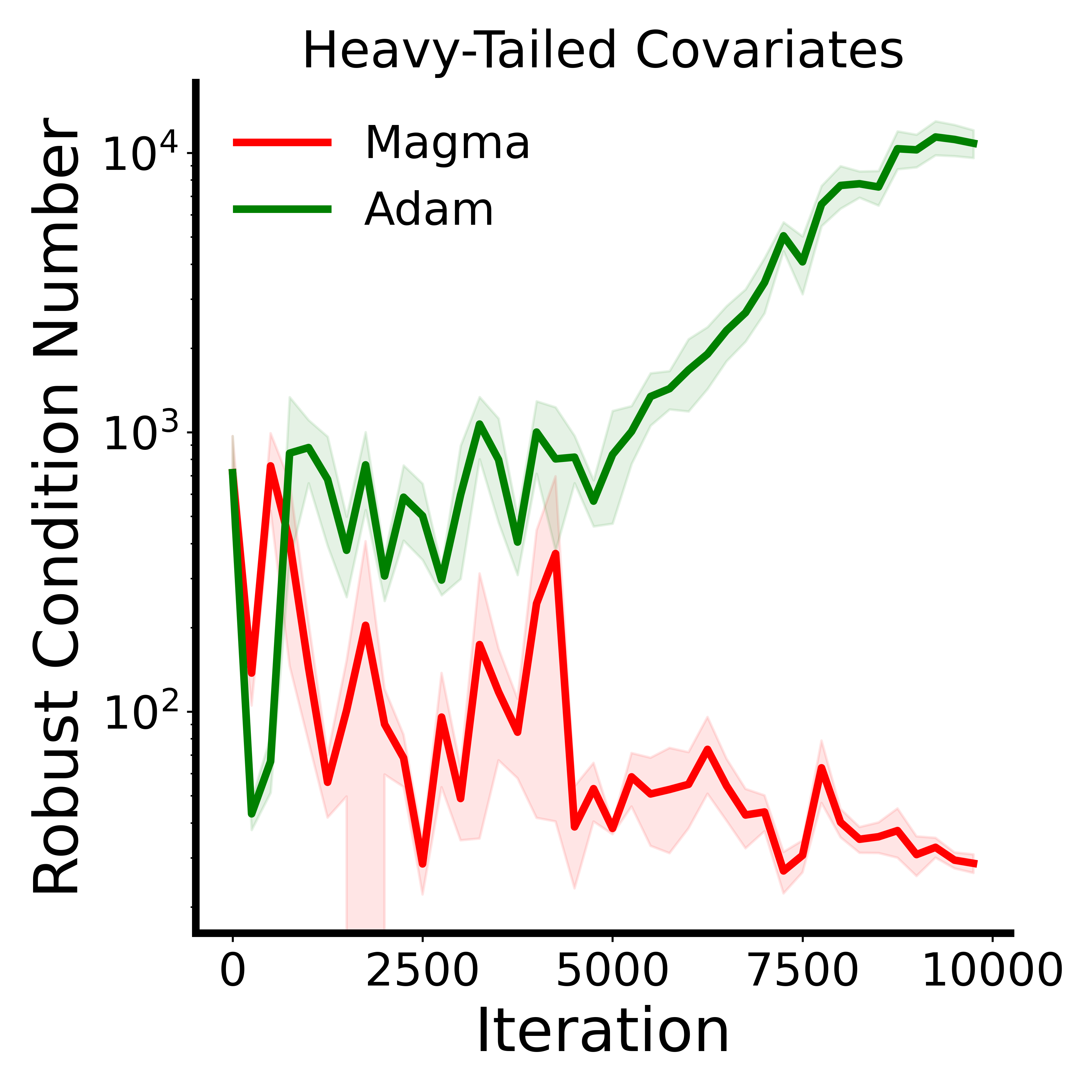}
\end{center}
\vskip -0.1in
\caption{\textit{Magma on light-tailed and heavy-tailed data distributions.} 
\underline{\textbf{\emph{Top:}}} Optimization trajectories for Adam and Magma. 
\underline{\textbf{\emph{Bottom:}}} Robust condition number defined as the ratio between the maximum and median eigenvalues of the loss Hessian.  
}
\label{fig:heavy-tailed}
\vskip -0.1in
\end{figure}

\begin{figure*}
\vskip -0.06in
\begin{center}
\includegraphics[width=0.98\textwidth]{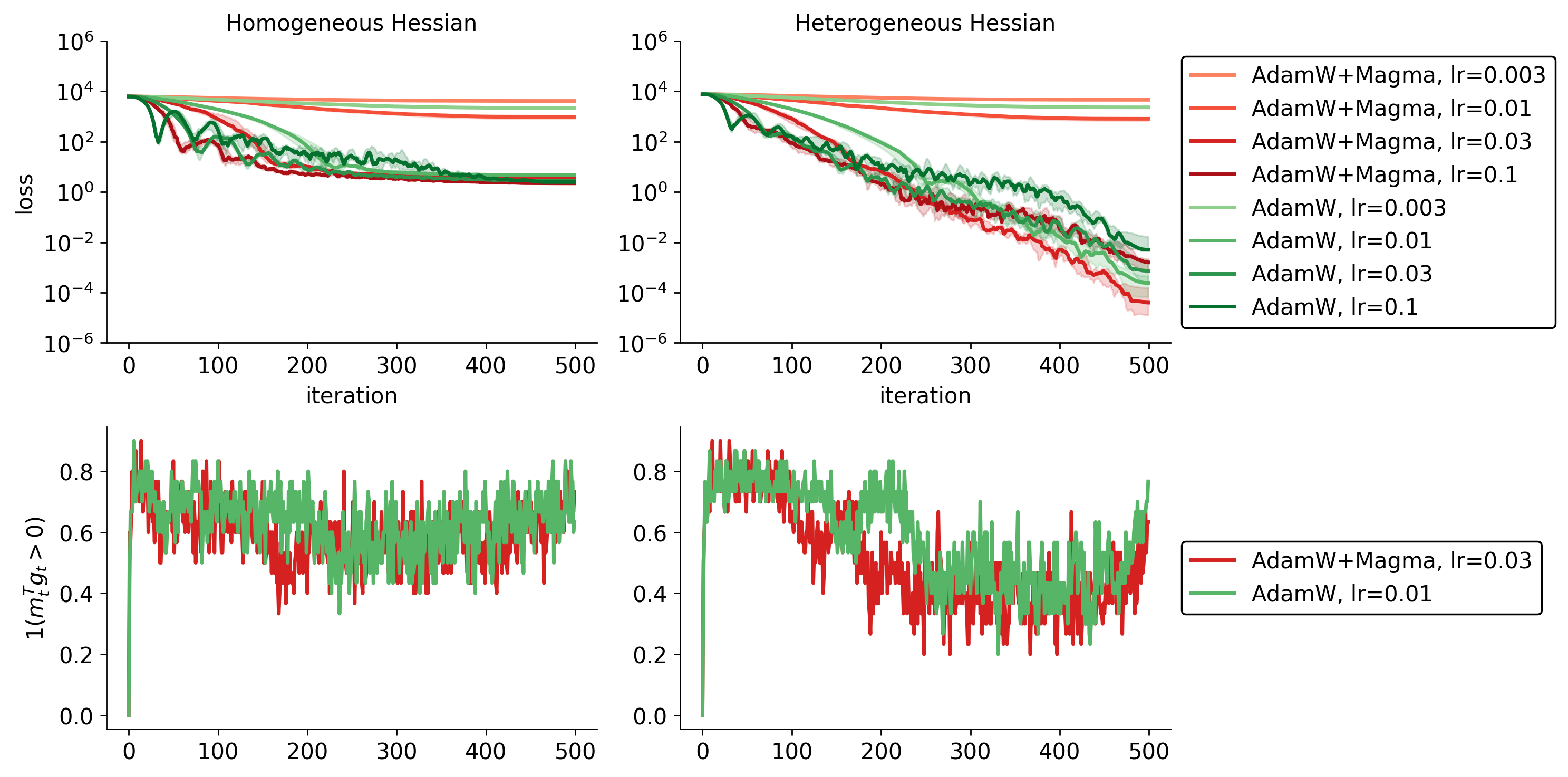}
\end{center}
\caption{\textit{Magma on homogeneous and heterogeneous quadratics.} 
\underline{\textbf{\emph{Top:}}} Optimization trajectories for AdamW and Magma on quadratic objectives with identical eigenspectra but different block structures. 
\underline{\textbf{\emph{Bottom:}}} Average gradient–momentum alignment per block. 
}
\label{fig:quadratic}
\vskip -0.1in
\end{figure*}

\subsection{Magma under Heavy-Tailed Gradient Noises}
A salient feature of training autoregressive language models is the presence of heavy-tailed stochastic gradient noise \citep{zhang2019gradient, kunstner2024heavy}. 
To isolate and study the effect of heavy-tailed noise on optimization dynamics, we evaluate Magma using the controlled benchmark proposed in  \citet{ahn2023linear}, which provides an abstraction of transformer training while faithfully reproducing key empirical phenomena observed in practice, including the performance gap between Adam and SGD for LLM pre-training.
The benchmark considers training a simplified linear transformer to solve a random linear regression task in a meta learning format \citep{garg2022can,von2023transformers,akyurek2022learning}. 
In this benchmark, we consider two data regimes, namely, \textit{light-tailed} and \textit{heavy-tailed} settings. Under the heavy-tailed setting, we modify the input sampling distribution to amplify tail behavior in the covariates, thereby inducing heavy-tailed stochastic gradient noise during optimization. We refer to Appendix \ref{appx:ahn_linear_benchmark} for full experimental details.

Figure \ref{fig:heavy-tailed} (top) compares the optimization trajectories of Adam and Magma under normal and heavy-tailed covariates. 
While both methods perform similarly under normal noise, Magma significantly outperforms Adam in the heavy-tailed setting, offering insight into its strong empirical performance in LLM pre-training, where heavy-tailed data is ubiquitous \citep{kunstner2024heavy,ahn2023linear}. 
Notably, this result is appealing given Adam’s known robustness to heavy-tailed distributions.

To further analyze this behavior, Figure \ref{fig:heavy-tailed} (bottom) reports the robust condition number along training trajectories \citep{jiang2023does}. Under heavy-tailed noise, Magma consistently attains substantially smaller condition numbers, indicating that its updates remain confined to well-conditioned regions of the loss landscape. This reflects the curvature-dependent geometric regularization induced by scaled random masking, which selectively suppresses curvature- and noise-dominated directions, thereby enhancing robustness to extreme gradient fluctuations.

\subsection{Magma on Heterogeneous Quadratics}
To further validate the effectiveness of Magma in heterogeneous landscapes, we evaluate Magma on a controlled quadratic benchmark adopted in recent works \citep{zhang2024transformers, orvieto2025search}.
The benchmark consists of two quadratic objectives with \emph{identical eigenspectra} but different block-wise curvature structure. In both cases, the loss takes the form $l(\vw) = \tfrac{1}{2} \vw^\top \mH \vw,$ where $\mH$ has eigenvalues spanning three orders of magnitude. The key distinction lies in how these eigenvalues are arranged: in the \emph{homogeneous} setting, eigenvalues of similar scale are grouped within blocks, whereas in the \emph{heterogeneous} setting, each block mixes eigenvalues with vastly different magnitudes, inducing strong curvature misalignment. 
Despite its simplicity, this heterogeneous structure qualitatively mimics the loss geometry observed in autoregressive transformers, in contrast to the more scale-separated curvature typical of CNNs. Full details are provided in Appendix \ref{appx:hetero_quad_setup}.

Figure~\ref{fig:quadratic} (top) presents optimization trajectories for both Hessian structures.
On the homogeneous problem, Magma and AdamW exhibit comparable performance, with Magma converging slightly faster in the early stages.  
However, in the heterogeneous problem, while AdamW is known to substantially outperform non-adaptive methods such as SGD \citep{zhang2024transformers,orvieto2025search}, Magma achieves faster convergence and a lower final loss than AdamW.  
This mirrors the gains observed in the pre-training experiments (cf. Table \ref{table:benchmark_c4}), suggesting that Magma is particularly effective in regimes where curvature is both ill-conditioned and misaligned across parameter subspaces.

Importantly, this advantage does not extend to architectures whose curvature resembles the homogeneous case. 
For example, when applied to ResNet-50 on CIFAR-10 classification, Magma does not improve over AdamW (94.46\% vs.\ 93.82\% test accuracy after 100 epochs with carefully tuned configurations), reinforcing the hypothesis that its benefits are specific to transformer-like loss geometry.

Figure~\ref{fig:quadratic} (bottom) analyzes the average gradient–momentum alignment within each block for AdamW and Magma.  
As expected, the homogeneous Hessian exhibits higher alignment across all blocks, reflecting its more benign conditioning.
Notably, although Magma explicitly suppresses updates that conflict with accumulated momentum, it does not significantly increase the alignment score itself.
This indicates that Magma’s gains arise not from altering momentum statistics, but from enforcing consistency between instantaneous gradients and long-term descent directions.

% % ========================================================================
% % ========================================================================
% % ========================================================================
% %         Discussion
% % ========================================================================
% % ========================================================================
% % ========================================================================
\section{Discussion} \label{sec:convergence_analysis}
We analyze the effect of scaled random masking in Magma through the lens of classical optimization theory. 

\textbf{Setup.}
We consider $\theta \in \R^d$, which can be decomposed as $\theta = (\theta^{(1)}, \theta^{(2)}, \cdots, \theta^{(B)})$ with $\theta^{(b)} \in \R^{d'}$ and $B d' = d$. We assume the objective is lower bounded: $l_* \triangleq \min_\theta l(\theta) > - \infty$. 
Let $\gM_t(\theta) \triangleq (\gM_t(\theta)^{(1)}, \cdots, \gM_t(\theta)^{(B)})$ be a scaled masking operator defined as $\gM_t(\vg)^{(b)} \triangleq \frac{m_{t}^{(b)}}{p} \vg^{(b)}$, where $m_{t}^{(b)} \sim \mathrm{Bernoulli}(p)$. 
We define an alignment score vector $\vs_t \in \R^B \triangleq (s_t^{(1)}, \cdots, s_t^{(B)})$. Also, we use an operation $ \vs_t \otimes \mI_{d'}$, where $\otimes$ is the Kronecker product and $\mI_{d'}$ is the $d' \times d'$ identity matrix, to scale each block update $\gM_t(g)^{(b)}$ by $s_t^{(b)}$ through $(\vs_t \otimes \mI_{d'}) \gM_t(\theta)$. For brevity, we denote $\mS_t \triangleq \vs_t \otimes \mI_{d'}$.

Throughout this section, we consider a constant learning rate SGD, isolating the impact of stochasticity and masking from the adaptivity of Adam-type methods: (Vanilla SGD) $\theta_{t+1} = \theta_t - \eta \vg_t$ for $t = 0, 1, \cdots$; (Magma) $\theta_{t+1} = \theta_t - \eta \mS_t \gM_t(\vg_t)$ for $t = 0, 1, \cdots$.
Extensions to adaptive optimizers follow the similar descent lemma framework and do not alter the core argument; see \citet{wang2022divergence,crawshaw2022robustness} for related analyses.
All proofs of claims are presented in Appendix \ref{appx:proofs}.

We begin by defining technical assumptions for the analysis.

\begin{assumption} \label{assumption:layerwise_smooth}
    There exist constants $L^{(b)} \geq 0$ such that for all $\theta \in \R^d$, $b \in [B]$, $\vu \in \R^{d'}$, $l(\theta + \mU_b \vu) \leq l(\theta) + \vu^\top \nabla_b l(\theta) + \frac{L^{(b)}}{2} \| \vu  \|^2,$
    where $\mU_b \vu \in \R^d$ denotes the vector with block $b$ equal to $\vu$ and other blocks 0.  
\end{assumption}

In Assumption \ref{assumption:layerwise_smooth}, the block-wise smoothness bound naturally captures heterogeneous nature of the transformer loss landscape. We define the block-wise smoothness weighted semi-norm as $\| \vg_t \|_L^2 \triangleq \sum_{b=1}^{B} L^{(b)} \| \vg_t^{(b)} \|^2$.

\begin{assumption} \label{assumption:layerwise_second_moment}
    There exist constants $\sigma_b \geq 0$ such that for all $\theta \in \R^d$ and $b \in [B]$, $\E [ \| \vg^{(b)}(\theta) \|^2 ] \leq \| \nabla_b l(\theta) \|^2 + \sigma_b^2,$ where $\vg^{(b)}(\theta)$ denotes the stochastic gradient for $\theta^{(b)}$.
\end{assumption}

Assumption \ref{assumption:layerwise_second_moment} is standard in stochastic optimization and captures bounded variance of the stochastic gradient.

\textbf{Descent lemma.} 
The following descent lemma characterizes the per-iteration decrease of a smooth objective.

\begin{lemma} \label{prop:magma_descent}
    Under Assumption \ref{assumption:layerwise_smooth}, Magma satisfies for all $t$, 
    \begin{equation} \label{eq:one_step_descent_magma}
        \E_t[l(\theta_{t+1})] 
            \leq l(\theta_t) - \eta \E_t [\vg_t^\top \mS_t \nabla l(\theta_t)] + \frac{\eta^2}{2p} \E_t [\| \mS_t \vg_t \|_L^2 ].
    \end{equation}
\end{lemma}

Compared to the vanilla SGD's bound 
$\E_t[l(\theta_{t+1})]  \leq l(\theta_t) - \eta \| \nabla l(\theta_t) \| + \frac{\eta^2}{2} \E_t [\| \vg_t \|_L^2 ]$ (obtained by $p=1$ and $\mS_t = \mI_{d}$), Magma differs from two terms: 1) the first-order term is reduced $\vg_t^\top \mS_t \nabla l(\theta_t) \le  \vg_t^\top \nabla l(\theta_t)$ a.e. since $\mS_t\preceq I_d$;
2) quadratic penalty uses effective smoothness constants. To see this, we define the second-moment scaling factor: $\rho_t^{(b)} \triangleq \frac{\E_t [\| s_t^{(b)} \vg_t^{(b)} \|^2] }{ \E_t [ \| \vg_t^{(b)} \|^2]}.$
Then, the quadratic penalty in \eqref{eq:one_step_descent_magma},  
\begin{equation} \label{eq:eff_smoothness}
    \frac{\eta^2}{2p} \E_t [\| \mS_t  \vg_t \|_L^2 ] 
        = \frac{\eta^2}{2} \sum_{b=1}^{B} \frac{\rho_t^{(b)} L^{(b)}}{p} \E_t [\| \vg_t^{(b)} \|^2 ], 
\end{equation}
describes how Magma rescales each block-wise smoothness constant as $L^{(b)} \mapsto \tilde{L}_{t}^{(b)} \triangleq  \frac{\rho_t^{(b)} L^{(b)}}{p}$. 
Accordingly, we define $\| \vg_t \|_{\tilde{L}_t}^2 \triangleq \sum_{b=1}^{B} \tilde{L}_t^{(b)} \| \vg_t^{(b)} \|^2$ and $\sigma^2_{\tilde{L}_t} \triangleq \sum_{b=1}^{B} \tilde{L}_{t}^{(b)} \sigma_b^2 $.

\textbf{Convergence analysis.} 
To derive a global convergence rate, we require a lower bound on the descent term in \eqref{eq:one_step_descent_magma}. The next lemma provides such a bound and defines the effective descent efficiency factors.

\begin{lemma} \label{sup_lem:descent_lower_bound}
    Under Assumption \ref{assumption:layerwise_second_moment}, it holds for all $t$,
    \begin{equation} \label{eq:lb_descent}
        \E [\vg_t^\top \mS_t  \nabla l(\theta_t)] \geq \| (\alpha_t \otimes \mI_{d'}) \nabla l (\theta_T) \|^2 - \sigma^2_{C_t},
    \end{equation}
    where $\alpha_t \triangleq (\alpha_t^{(1)}, \cdots, \alpha_t^{(B)})$ and $\sigma^2_{C_t} \triangleq \sum_{b=1}^{B} c_t^{(b)}  \sigma^2_b$; $\alpha_t^{(b)} \in [\sigmoid(-1/\tau), \sigmoid(1/\tau)] $ and $c_t^{(b)} \in [0, \sigmoid(1/\tau) / 2]$ are constants defined in proof.
\end{lemma}

Lemma \ref{sup_lem:descent_lower_bound} yields the effective descent efficiency under the Magma's scaled random masking. Specifically, $\alpha_t^{(b)}$ quantifies the fraction of descent preserved on block $b$ at iteration $t$ (up to the noise-coupling term).
To describe aggregated impacts, we define $\bar{\alpha}_{T}^{\textrm{eff}} \triangleq \frac{\sum_{t=0}^{T-1} \E [ \| (\alpha_t \otimes \mI_{d'}) \nabla l (\theta_t) \|^2 ] ]}{\sum_{t=0}^{T-1} \E [ \| \nabla l(\theta_t) \|^2 ]}$. Also, we define an average noise-descent coupling $\sigma^2_{\bar{C}} \triangleq \frac{1}{T}\sum_{t=0}^{T-1}\E[\sigma^2_{C_t}]$ and the maximum effective smoothness $\tilde{L}_t^{\max} \triangleq \max_{b\in[B]} \tilde{L}_{t}^{(b)}$.

\begin{theorem} \label{theorem:main_result}
    Under Assumptions \ref{assumption:layerwise_smooth} and \ref{assumption:layerwise_second_moment}, for the stepsize $\eta \in (0, \frac{\bar{\alpha}_{T}^{\textrm{eff}}}{ \tilde{L}_t^{\max} } ]$, it holds that 
    % \todo{[Double check if $p$ is included or not.]} \todo{Define $\bar{\sigma}_{\tilde{L}}$}
    \begin{equation} \label{eq:magma:conv}
        \frac{1}{T} \sum_{t=0}^{T-1} \E \left[ \| \nabla l (\theta_t) \|^2 \right]
            \le \frac{2(l(\theta_0) - l_*) }{\eta \bar{\alpha}_{T}^{\textrm{eff}} T}
            +   \frac{2 \sigma^2_{\bar{C}}}{\bar{\alpha}_{T}^{\textrm{eff}}}  
            + \frac{\eta \bar{\sigma}^2_{\tilde{L}}}{ \bar{\alpha}_{T}^{\textrm{eff}}},
    \end{equation}
    where $ \bar{\sigma}^2_{\tilde{L}} = \sum_{t=0}^{T-1} \frac{1}{T} \E [\sigma^2_{\tilde{L}_t}]$. 
\end{theorem}

Theorem \ref{theorem:main_result} provides the standard constant-step nonconvex stationarity guarantee with explicit dependence on the effective curvature constants.
Note that for SGD, \eqref{eq:magma:conv} reduces to $\frac{1}{T} \sum_{t=0}^{T-1} \E \| \nabla l(\theta_t ) \|^2 \le \frac{2 (l(\theta_0) - l_*)}{\eta T} + \eta \sigma_L^2, \; \text{for } \eta \in (0, \frac{1}{L_{\max}}]$ with $L_{\max} \triangleq \max_b L^{(b)}$.
In this regard, the scaling factor $s_t$ impacts the \emph{descent efficiency} $\bar{\alpha}_T^{\mathrm{eff}}$, the \emph{effective noise level} $\bar{\sigma}^2_{\tilde{L}}$, and the average noise-descent coupling $\sigma^2_{\bar{C}}$.
On the one hand, reducing $s_t^{(b)}$ decreases the descent term through $\alpha_t^{(b)}$,
potentially slowing optimization.
On the other hand, the same scaling suppresses the curvature-weighted noise contribution via the effective smoothness constants $\tilde{L}_t^{(b)} = \frac{\rho_t^{(b)}}{p} L^{(b)}$, thereby reducing the stationary noise floor.
In this regard, Theorem \ref{theorem:main_result} shows that scaling is not uniformly beneficial. Rather, it makes explicit that any improvement must come from \textit{where} the suppression occurs (which blocks are attenuated), not merely how much suppression is applied on average.

\textbf{On effective iteration number.}
A favorable regime arises when $\rho_t^{(b)} \ll 1$ holds for blocks with large curvature $L^{(b)}$ (or more generally, for blocks dominating $\sum_b L^{(b)} \E_t |\vg_t^{(b)}|^2$).
In this case, both the effective smoothness constants $\tilde L_t^{(b)}$ and the curvature-weighted noise $\sigma^2_{\tilde L_t}$ are reduced, which simultaneously (i) enlarges the admissible stepsize range and (ii) lowers the stationary error floor in Theorem \ref{theorem:main_result}.
Moreover, the enlarged stability region increases the number of iterations for which the stochastic update satisfies the descent surrogate $\eta \vg_t^\top \mS_t\nabla l(\theta_t) - \frac{\eta^2}{2p}\|\mS_t \vg_t\|_L^2 > 0$ thereby accelerating convergence in terms of effective progress per iteration.

In this regard, Magma enlarges the stability region by selectively suppressing blocks that dominate the curvature-weighted noise and smoothness constraints. 
In ill-conditioned and highly heterogeneous transformer landscapes \citep{pan2023toward,ahn2023linear,orvieto2025search}, stability is typically governed by a small subset of high-curvature or high-variance blocks. 
By attenuating precisely these blocks according to the momentum–gradient alignment, Magma reduces the effective smoothness $\tilde L_t$ and noise level $\sigma^2_{\tilde L_t}$ that limit the admissible stepsize. 
This targeted reduction explains the empirically observed widening of the stable learning-rate regime (Figure \ref{fig:learning-rate-sensitivity} in Appendix) and accounts for the effectiveness of structured scaling over uniform or random masking in large-scale LLM training.

% ========================================================================
% ========================================================================
% ========================================================================
%         Literature
% ========================================================================
% ========================================================================
% ========================================================================

\section{Literature Review}
\textbf{Stabilizing LLM training.}
A line of recent work addresses instability in LLM training by explicitly constraining optimizer updates in unstable regimes.
The method most closely related to Magma is Cautious Optimizer \citep{liang2024cautious} where an update for each parameter is masked when its gradient has an opposite direction to its first-moment estimate for ensuring a descent direction update under a small step size.
However, due to its deterministic masking rule, Cautious Optimizer lacks the geometric regularization effect promoting flatter optimization trajectories, unlike Magma.
More broadly, peculiar optimization landscape of large-scale transformers is known to contain many sources of training instability such as loss spikes and gradient explosions \citep{chowdhery2023palm}.
To mitigate these issues, prior work has explored gradient clipping with a periodic momentum reset to address negative impacts of gradient spikes \citep{huang2025spam}, initialization methods for stable training dynamics \citep{nguyen2019transformers,takase2023spike}, and architectural interventions that reshape gradient propagation \citep{dettmers20218,xiong2020layer,wang2024deepnet}.
Despite this wide exploration, the direction of inducing geometric regularization through stochastic perturbations of the optimization process remains relatively underexplored.

\textbf{Geometry-aware and trust-region methods.}
A large body of work improves optimization by incorporating local geometry through curvature-aware updates \citep{dauphin2014identifying, martens2015optimizing, gupta2018shampoo,  yao2021adaHessian} or trust-region constraints \citep{foret2020sharpness,kwon2021asam}. 
The former direction aims to approximate second-order structure by tractable preconditioners, recently extending to efficient eigen-decomposition for LLM training \citep{vyas2024soap}. 
On the contrary, the latter trust-region–flavored methods, such as SAM and its recent variants \citep{mueller2023normalization,li2024friendly}, bias optimization toward flatter regions by optimizing robustness to parameter perturbations, albeit at the cost of additional gradient evaluations. 
Instead of computing explicit curvature matrices or adversarial perturbation, Magma penalizes sharp curvature along its own block-wise update directions via stochastic masking, offering a lightweight alternative to flatness-seeking optimization.

\textbf{Stochastic perturbation and noise injection.}
Stochastic perturbation is a well-established strategy in machine learning, ranging from random gradient masking in meta-learning \citep{tseng2020regularizing} and federated learning \citep{wei2020framework} to the annealed Gaussian noise used in Bayesian inference \citep{welling2011bayesian, neelakantan2015adding}. 
A prominent example is Dropout \citep{srivastava2014dropout} that randomly masks hidden units, which has been shown to induce a data-dependent weight-space regularizer that promotes model stability  \citep{mianjy2018implicit,wei2020implicit,zhang2024implicit}.
In the context of modern LLMs, this principle has evolved into embedding-space perturbations, injecting structured noise into token embeddings during instruction fine-tuning \citep{jain2024neftune}. 
However, the impact of structured and stateful perturbations on optimization dynamics remains relatively under-explored. 
Magma operates directly on parameter updates and their alignment with local curvature for regularizing the optimization dynamics, highlighting a role for stochastic perturbation particularly suited to modern large-scale optimization.

% ========================================================================
% ========================================================================
% ========================================================================
%         Conclusion
% ========================================================================
% ========================================================================
% ========================================================================
\section{Conclusion}
We showed that randomly masking parameter updates can substantially improve LLM pre-training, which smooths optimization trajectories with an implicit curvature-dependent geometric regularization. 
The proposed Magma, which leverages momentum-gradient alignment to further enhance masked updates, achieves consistent improvements over state-of-the-art adaptive optimizers with negligible overhead. 
These results challenge the prevailing assumption that dense updates are inherently optimal for backpropagation-based neural net training. 
We expect this perspective to inspire new classes of optimization algorithms that exploit structured stochasticity to improve both optimization stability and generalization in training large-scale foundation models having ill-conditioned and highly heterogeneous optimization landscapes.

\section*{Impact Statement}
This paper presents work whose goal is to advance the field of Machine Learning. There are many potential societal consequences of our work, none which we feel must be specifically highlighted here.

\bibliography{icml2026}
% \bibliographystyle{icml2026}

%%%%%%%%%%%%%%%%%%%%%%%%%%%%%%%%%%%%%%%%%%%%%%%%%%%%%%%%%%%%%%%%%%%%%%%%%%%%%%%
%%%%%%%%%%%%%%%%%%%%%%%%%%%%%%%%%%%%%%%%%%%%%%%%%%%%%%%%%%%%%%%%%%%%%%%%%%%%%%%
% APPENDIX
%%%%%%%%%%%%%%%%%%%%%%%%%%%%%%%%%%%%%%%%%%%%%%%%%%%%%%%%%%%%%%%%%%%%%%%%%%%%%%%
%%%%%%%%%%%%%%%%%%%%%%%%%%%%%%%%%%%%%%%%%%%%%%%%%%%%%%%%%%%%%%%%%%%%%%%%%%%%%%%
\newpage
\appendix
\onecolumn

\setcounter{table}{0}
\renewcommand{\thetable}{A\arabic{table}}

\setcounter{figure}{0}
\renewcommand{\thefigure}{A\arabic{figure}}

\section{Proofs of Claims} \label{appx:proofs}

\subsection{Proof of Proposition \ref{lemma:implicit_reg}} \label{appx:proof_prop1}

\begin{proof}
Condition on $\mathcal F_t$, under which $\theta_t$ and $\Delta_t$ are deterministic.
A second-order Taylor expansion of $l(\theta_t-\tilde{\Delta}_t)$ around $\theta_t$ yields
\begin{equation} \label{eq:taylor}
    l(\theta_t-\tilde{\Delta}_t) = l(\theta_t) - \sum_{b=1}^B (\vg_t^{(b)})^\top \tilde{\Delta}_t^{(b)} + \frac{1}{2} \sum_{b=1}^B\sum_{b'=1}^B (\tilde{\Delta}_t^{(b)})^\top \mathbf{H}_{bb'}(\theta_t) \tilde{\Delta}_t^{(b')} + R_2(\tilde{\Delta}_t),
\end{equation}
where $\vg_t=\nabla l(\theta_t)$, $\mathbf{H}_{bb'}(\theta_t)$ is the $(b,b')$ Hessian block, and the remainder satisfies $|R_2(\tilde{\Delta}_t)| = O \left(\sum_{b=1}^B \|\tilde{\Delta}_t^{(b)}\|^3\right).$

By construction of the masked update and conditioning on $(\theta_t,\Delta_t)$,
\begin{equation}
    \mathbb{E}[\tilde{\Delta}_t^{(b)} \mid \theta_t,\Delta_t] = \frac{1}{p}\mathbb{E}[m_{t}^{(b)}]\Delta_t^{(b)} = \Delta_t^{(b)}.
\end{equation}

Moreover, using independence of $\{m_{t}^{(b)}\}_{b=1}^B$,
\begin{equation}
    \mathbb{E} \left[ (\tilde{\Delta}_t^{(b)})^\top \mathbf{H}_{bb'}(\theta_t) \tilde{\Delta}_t^{(b')} \mid \theta_t,\Delta_t \right] =
        \begin{cases}
        (\Delta_t^{(b)})^\top \mathbf{H}_{bb'}(\theta_t)\Delta_t^{(b')},          & b\neq b',\\[6pt]
        \frac{1}{p} (\Delta_t^{(b)})^\top \mathbf{H}_{bb}(\theta_t)\Delta_t^{(b)},
        & b=b'.
        \end{cases}
\end{equation}
Taking conditional expectations in \eqref{eq:taylor} and regrouping terms, we obtain
\begin{equation}
    \mathbb{E}[l(\theta_t-\tilde{\Delta}_t)\mid \theta_t,\Delta_t] = l(\theta_t-\Delta_t) + \sum_{b=1}^B \frac{1-p}{2p} (\Delta_t^{(b)})^\top \mathbf{H}_{bb}(\theta_t) \Delta_t^{(b)}  + \mathbb{E}[R_2(\tilde{\Delta}_t)\mid \theta_t,\Delta_t].
\end{equation}

Finally, since $\mathbb{E}\|\tilde{\Delta}_t^{(b)}\|^3 = O(\|\Delta_t^{(b)}\|^3)$ for each $b$, the expected remainder term is of order $O \left(\sum_{b=1}^B \|\Delta_t^{(b)}\|^3\right)$, completing the proof.
\end{proof}

\subsection{Proof of Lemma \ref{prop:magma_descent}}
\begin{proof}
    
    Under Assumption \ref{assumption:layerwise_smooth}, for any $\theta$ and any multi-block $\triangle = (\triangle^{(1)}, \cdots, \triangle^{(B)})$, it holds
    \begin{equation} \label{sequential_descent}
        l(\theta + \triangle) \leq l(\theta) + \triangle^\top \nabla l(\theta) + \frac{1}{2} \| \triangle \|_L^2,
    \end{equation}
    which can be shown by applying Assumption \ref{assumption:layerwise_smooth} sequentially.

    Applying \eqref{sequential_descent} with $\triangle = -\eta \mS_t \gM_t(\theta)$ yields
    \begin{equation} \label{intem9}
        l(\theta_{t+1})
            \leq l(\theta_t) - \eta (\mS_t \gM_t(\vg_t))^\top \nabla l(\theta_t) + \frac{\eta^2}{2} \| \mS_t \gM_t(\vg_t) \|_L^2.
    \end{equation}
    
    By the independence between $m_{t}^{(b)}$ and $\vg_t^{(b)}$, we have 
    \begin{equation}
        \E_t [(\mS_t \gM_t(\vg_t))^\top \nabla l(\theta_t)] = \E [\vg_t^\top \mS_t \nabla l(\theta_t)],
    \end{equation}
    and
    \begin{equation}
        \E_t [\| \mS_t \gM_t(\vg_t) \|_L^2] = \frac{1}{p} \E_t [\| \mS_t \vg_t \|_L^2].
    \end{equation}
    
    Therefore, taking conditional expectation $\E_t [\cdot]$ on both side of \eqref{intem9} yields the desired result.
\end{proof}

\subsection{Proof of Lemma \ref{sup_lem:descent_lower_bound}}

\begin{proof}
    We start by defining the event
    \begin{equation}
        E_t^{(b)} \triangleq \{ \text{cos}(\vg_t^{(b)}, \nabla_b l(\theta_t)) \geq \gamma^{(b)} \}
    \end{equation} for an alignment threshold $\gamma^{(b)} \in (0, 1]$.
    We let $e_t^{(b)} \triangleq \mathbb{P}(E_t^{(b)} | \gF_t)$. Also, we define constants $s_- \triangleq \sigmoid(- 1 / \tau)$, $s_+ \triangleq \sigmoid(1 / \tau)$, and $s_{\gamma}^{(b)} \triangleq \sigmoid(\gamma^{(b)} / \tau)$ for $b \in [B]$. 

    By using the law of total expectation, we get
    \begin{multline} \label{intem1}
        \E [\vg_t^\top \mS_t \nabla l(\theta_t)] 
            = \sum_{b=1}^{B} \E [s_t^{(b)} (\vg_t^{(b)})^\top \nabla_b l(\theta_t)] \\ 
            = \sum_{b=1}^{B}  \left\lbrace e_t^{(b)} \E [s_t^{(b)} (\vg_t^{(b)})^\top \nabla_b l(\theta_t) | E_t^{(b)}] + (1-e_t^{(b)}) \E [s_t^{(b)} (\vg_t^{(b)})^\top \nabla_b l(\theta_t) | (E_{t}^{(b)})^c] \right\rbrace \\
            \geq \sum_{b=1}^{B}  \left\lbrace e_t^{(b)} \E [s_{\gamma}^{(b)} (\vg_t^{(b)})^\top \nabla_b l(\theta_t) | E_t^{(b)}] + (1-e_t^{(b)}) \E [s_- (\vg_t^{(b)})^\top \nabla_b l(\theta_t) | (E_{t}^{(b)})^c] \right\rbrace,
    \end{multline}
    where the inequality comes from $s_t^{(b)} \geq s_{\gamma}^{(b)}$ on $E_t^{(b)}$ and $s_t^{(b)} \geq s_-$ on $(E_{t}^{(b)})^c$ for all $b$. 
    
    Then, using unbiasedness of $\vg_t$, we get 
    \begin{equation} 
        \E [\vg_t^\top \mS_t  \nabla l(\theta_t)] 
            \ge \sum_{b=1}^{B} \left\lbrace s_{\gamma}^{(b)} \| \nabla l (\theta_T) \|^2 +  (s_- - s_{\gamma}^{(b)}) \E [(\vg_t^{(b)})^\top \nabla_b l(\theta_t) \mathbf{1}_{(E_{t}^{(b)})^c}] \right\rbrace.
    \end{equation}
    
    On $(E_{t}^{(b)})^c$, we have $(\vg_t^{(b)})^\top \nabla_b l(\theta_t) < \gamma^{(b)} \| \nabla_b l(\theta_t) \| \| \vg_t^{(b)} \|$ whenever $ \| \nabla_b l(\theta_t) \| \| \vg_t^{(b)} \| > 0$. Therefore, taking conditional expectation and applying Cauchy-Schwarz gives
    \begin{multline} \label{intem2}
        \E [(\vg_t^{(b)})^\top \nabla_b l(\theta_t) \mathbf{1}_{(E_{t}^{(b)})^c}]
            \le \gamma^{(b)} \| \nabla_b l(\theta_t) \| \; \E [\|\vg_t^{(b)} \| \mathbf{1}_{(E_{t}^{(b)})^c} ] \\
            \le \gamma^{(b)} \| \nabla_b l(\theta_t) \| \sqrt{(1 - e_t^{(b)}) \E \|\vg_t^{(b)} \|^2}  
            \le \gamma^{(b)} \| \nabla_b l(\theta_t) \| \sqrt{(1 - e_t^{(b)}) (\|\nabla_b l(\theta_t) \|^2 + \sigma^2_b)},
    \end{multline}
    where the last inequality holds due to Assumption \ref{assumption:layerwise_second_moment}.
    
    Therefore, since $s_- - s_{\gamma}^{(b)} < 0$, combining \eqref{intem1} and \eqref{intem2} gives 
    \begin{equation} \label{intem3}
        \E [\vg_t^\top \mS_t \nabla l(\theta_t)] \ge 
        \sum_{b=1}^{B} \left\lbrace s_{\gamma}^{(b)} \| \nabla_b l (\theta_T) \|^2  
        +  (s_- - s_{\gamma}^{(b)}) \gamma^{(b)} \| \nabla_b l(\theta_t) \| \sqrt{(1 - e_t^{(b)}) (\|\nabla_b l(\theta_t) \|^2 + \sigma^2_b)} \right\rbrace.
    \end{equation}
    
    Finally, note that for $a, b\ge 0$, it holds $ \sqrt{a (a+b)} \le a + \frac{b}{2}$. Applying this inequality with $a = \| \nabla_b l(\theta_t) \|^2$ and $b = \sigma^2_b$ to \eqref{intem3} yields 
    \begin{multline} \label{intem4}
        \E [\vg_t^\top \mS_t \nabla l(\theta_t)] \ge 
        \sum_{b=1}^{B} \left\lbrace 
            \left(s_{\gamma}^{(b)} - (s_{\gamma}^{(b)} - s_-) \gamma^{(b)} \sqrt{1 - e_t^{(b)}} \right) \| \nabla_b l (\theta_T) \|^2  
        -  \frac{(s_{\gamma}^{(b)} - s_-)\gamma^{(b)}}{2} \sqrt{1 - e_t^{(b)}}  \sigma^2_b \right\rbrace \\ 
        = \sum_{b=1}^{B} \left\lbrace 
            \alpha_t^{(b)} \| \nabla_b l (\theta_T) \|^2  
        -  c_t^{(b)}  \sigma^2_b \right\rbrace,
    \end{multline}
    where $\alpha_t^{(b)} \triangleq s_{\gamma}^{(b)} - (s_{\gamma}^{(b)} - s_-) \gamma^{(b)} \sqrt{1 - e_t^{(b)}}$ and $c_t^{(b)} \triangleq \frac{(s_{\gamma}^{(b)} - s_-)\gamma^{(b)}}{2} \sqrt{1 - e_t^{(b)}}$. The interval bounds of $\alpha_t^{(b)}$ and $c_t^{(b)}$ follow directly from the definition. 
    
\end{proof}

\subsection{Proof of Theorem \ref{theorem:main_result}}
\begin{proof}
    First, we note from \eqref{eq:eff_smoothness} that  
    \begin{equation}
        \frac{1}{p} \E_t [\| \mS_t \vg_t \|^2_L] = \sum_{b=1}^{B} \tilde{L}_{t}^{(b)} \E_t[\| \vg_t^{(b)} \|^2] \le \sum_{b=1}^{B} \tilde{L}_{t}^{(b)} (\|\nabla_b l(\theta_t) \|^2  + \sigma_b^2), 
    \end{equation}
    where the inequality comes from Assumption \ref{assumption:layerwise_second_moment}.
    Then, applying \eqref{eq:lb_descent} in Lemma \ref{sup_lem:descent_lower_bound} to \eqref{eq:one_step_descent_magma} in the descent lemma yields
    \begin{multline} \label{intem11}
        \E_t[l(\theta_{t+1})] 
            \leq l(\theta_t) - \eta \E_t [\vg_t^\top \mS_t  \nabla l(\theta_t)] + \frac{\eta^2}{2p} \E_t [\| \mS_t \vg_t \|_L^2 ] \\
                \leq l(\theta_t) - \eta \left( \| (\alpha_t \otimes \mI_{d'}) \nabla l (\theta_T) \|^2 - \sigma^2_{C_t} \right) + \frac{\eta^2}{2} \sum_{b=1}^{B} \tilde{L}_{t}^{(b)} (\|\nabla_b l(\theta_t) \|^2  + \sigma_b^2).
    \end{multline}
    
    By taking total expectation and telescoping \eqref{intem11} from $t=0$ to $t = T-1$ for some $T \ge 1$, we get
    \begin{equation}
        \sum_{t=0}^{T-1} \eta \E [ \| (\alpha_t \otimes \mI_{d'}) \nabla l (\theta_T) \|^2 ] 
        \le l(\theta_0) - l_* 
            + \sum_{t=0}^{T-1} \eta \E[\sigma^2_{C_t}] 
            + \frac{\eta^2}{2} \sum_{t=0}^{T-1} \sum_{b=1}^{B} \tilde{L}_{t}^{(b)} (\|\nabla_b l(\theta_t) \|^2  + \sigma_b^2),
    \end{equation}
    where we use the factor that $\E [l(\theta_T)] > l_*$.

    Now, by doing elementary algebra, we get
    \begin{multline}
    \sum_{t=0}^{T-1} \eta \E [ \| (\alpha_t \otimes \mI_{d'}) \nabla l (\theta_T) \|^2 ] 
        \le l(\theta_0) - l_* 
            + \sum_{t=0}^{T-1} \eta \E[\sigma^2_{C_t}] 
            + \frac{\eta^2}{2} \sum_{t=0}^{T-1} \sum_{b=1}^{B} \tilde{L}_{t}^{(b)} (\|\nabla_b l(\theta_t) \|^2  + \sigma_b^2) \\
        \Rightarrow (\text{Dividing both side by } \eta     \bar{\alpha}_{T}^{\textrm{eff}} T) \\ 
    \frac{1}{T} \sum_{t=0}^{T-1} \E [ \| \nabla l(\theta_t) \|^2] 
        \le \frac{l(\theta_0) - l_* }{\eta \bar{\alpha}_{T}^{\textrm{eff}} T}
            +  \frac{\sum_{t=0}^{T-1} \E[\sigma^2_{C_t}]}{\bar{\alpha}_{T}^{\textrm{eff}} T}  
            + \frac{\eta}{2 \bar{\alpha}_{T}^{\textrm{eff}} T} \sum_{t=0}^{T-1} \sum_{b=1}^{B} \tilde{L}_{t}^{(b)} (\|\nabla_b l(\theta_t) \|^2  + \sigma_b^2) \\
        \Rightarrow (\text{Definitions of } \sigma^2_{\bar{C}})  \\
    \frac{1}{T} \sum_{t=0}^{T-1} \E [ \| \nabla l(\theta_t) \|^2] 
        \le \frac{l(\theta_0) - l_* }{\eta \bar{\alpha}_{T}^{\textrm{eff}} T}
            +   \frac{\sigma^2_{\bar{C}}}{\bar{\alpha}_{T}^{\textrm{eff}}}  
            + \frac{\eta}{2 \bar{\alpha}_{T}^{\textrm{eff}} T} \sum_{t=0}^{T-1} \sum_{b=1}^{B} \tilde{L}_{t}^{(b)} (\|\nabla_b l(\theta_t) \|^2  + \sigma_b^2).
    \end{multline}
    
    Rearranging term gives
    \begin{equation} \label{intem:rearrange}
        \frac{1}{T} \sum_{t=0}^{T-1} \left( \|\nabla l(\theta_t) \|^2
            - \frac{\eta}{2 \bar{\alpha}_{T}^{\textrm{eff}}} \sum_{b=1}^{B} \tilde{L}_{t}^{(b)}\|\nabla_b l(\theta_t) \|^2 \right) 
        \le \frac{l(\theta_0) - l_* }{\eta \bar{\alpha}_{T}^{\textrm{eff}} T}
            +   \frac{\sigma^2_{\bar{C}}}{\bar{\alpha}_{T}^{\textrm{eff}}}  
            + \frac{\eta}{2 \bar{\alpha}_{T}^{\textrm{eff}} T} \sum_{t=0}^{T-1} \sum_{b=1}^{B} \tilde{L}_{t}^{(b)} \sigma_b^2.
    \end{equation}
    
    Since $0 < \eta \leq \frac{\bar{\alpha}_{T}^{\textrm{eff}}}{ \tilde{L}_t^{\max}} $ by definition, we have
    \begin{multline} \label{intem:lr_apply}
        \frac{1}{T} \sum_{t=0}^{T-1} \left( \|\nabla l(\theta_t) \|^2
            - \frac{\eta}{2 \bar{\alpha}_{T}^{\textrm{eff}}} \sum_{b=1}^{B} \tilde{L}_{t}^{(b)}\|\nabla_b l(\theta_t) \|^2 \right) 
        \ge \frac{1}{T} \sum_{t=0}^{T-1} \left( \|\nabla l(\theta_t) \|^2
            - \frac{1}{2 \tilde{L}_t^{\max}} \sum_{b=1}^{B} \tilde{L}_{t}^{(b)}\|\nabla_b l(\theta_t) \|^2 \right)  \\
        \ge \frac{1}{T} \sum_{t=0}^{T-1} \left( \|\nabla l(\theta_t) \|^2
            - \frac{1}{2} \sum_{b=1}^{B} \|\nabla_b l(\theta_t) \|^2 \right) 
        = \frac{1}{2T} \sum_{t=0}^{T-1}  \|\nabla l(\theta_t) \|^2.
    \end{multline}
    
    Therefore, applying \eqref{intem:lr_apply} to \eqref{intem:rearrange} yields the following inequality as desired:
    \begin{multline}
        \frac{1}{T} \sum_{t=0}^{T-1}  \|\nabla l(\theta_t) \|^2 
            \le \frac{2(l(\theta_0) - l_*) }{\eta \bar{\alpha}_{T}^{\textrm{eff}} T}
            +   \frac{2 \sigma^2_{\bar{C}}}{\bar{\alpha}_{T}^{\textrm{eff}}}  
            + \frac{\eta}{\bar{\alpha}_{T}^{\textrm{eff}} T} \sum_{t=0}^{T-1} \sum_{b=1}^{B} \tilde{L}_{t}^{(b)} \sigma_b^2 
                = \frac{2(l(\theta_0) - l_*) }{\eta \bar{\alpha}_{T}^{\textrm{eff}} T}
                +   \frac{2 \sigma^2_{\bar{C}}}{\bar{\alpha}_{T}^{\textrm{eff}}}  
                + \frac{\eta \bar{\sigma}^2_{\tilde{L}}}{ \bar{\alpha}_{T}^{\textrm{eff}}} ,
    \end{multline}
    where $ \bar{\sigma}^2_{\tilde{L}} = \sum_{t=0}^{T-1} \frac{1}{T} \E [\sigma^2_{\tilde{L}_t}]$.
\end{proof}

\section{Experimental Details} 

\subsection{C4 Pre-Training Benchmark Setup} \label{appx:exp_details}
We follow the setup introduced in \citet{zhao2024galore}. Specifically, we use a fixed batch size of 512 and a max sequence length of 256 with a search grid of 1e-4, 5e-4, 1e-3, 5e-3, 1e-2 for the learning rate. For learning rate scheduling, we implement an initial warm-up for 10\% of the total steps, followed by a cosine annealing that decays the learning rate to 10\% of the peak value. We run 10K, 20K, 60K, and 100K iterations for the 60M, 130M, 350M, and 1B models, respectively, and report the final evaluation perplexity from the run with the best learning rate.

\subsection{Nano MoE Pre-Training Benchmark Setup} \label{appx:nano_moe-exp_details}
We follow the default setup presented in \citet{wolfe2024nanomoe}. 
Specifically, we use a GPT2 \citep{radford2019language}-style transformer with 124M number of parameters. 
Further for MoE configuration, the model uses 8 experts per MoE layer, with top-2 routing such that each token is dynamically dispatched to two experts. 
Further, the MoE layers are applied with a stride of 2, with dense and MoE layers alternating. The model also includes a number of training stabilization techniques, such as auxiliary load-balancing loss and switch transformer \citep{fedus2022switch}-style initialization. Refer to \citet{wolfe2024nanomoe} for full details. 
Finally, the model is trained for 50K iterations on 8xA100 GPUs using the default configuration: a batch size of 12, gradient accumulation of 40, a sequence length of 1024 tokens, a minimum learning rate of 5e-6, a weight decay of 0.1, and a grad clip norm of 1.0.

\subsection{Heterogeneous Quadratic Benchmark Setup} \label{appx:hetero_quad_setup}
We provide full details of the quadratic benchmark used in \citet{orvieto2025search}.
The setup consists of two quadratic optimization problems in $\R^9$, each defined by a Hessian matrix with an identical eigenspectrum and a $3 \times 3$ block-diagonal structure. 
Concretely, we consider losses of the form $L(w) = \frac{1}{2}w^\top Hw,$ where the eigenvalues of $H$ are $\{1, 2,3, 99, 100, 101, 4998, 4999, 5000\}$. 
While both problems share this eigenspectrums, they differ substantially in how the eigenvalues are arranged across blocks. 

In the homogeneous Hessian, eigenvalues are grouped by scale within each block: $\{1,2,3\}, \{99, 100, 101\}, \{4998, 4999, 5000 \}$.
In contracst, the heterogeneous Hessian interleaves eigenvalues of vastly different magnitudes within each block: $\{1,99, 4998\}, \{2, 100, 4999 \}, \{3, 101, 5000 \}$. 
This structural difference is intended to mimic qualitative distinctions between loss landscapes of autoregressive language models and those of more shallow architectures such as CNNs. 

The Hessian $H$ is constructed by first forming diagonal matrices with the specified eigenvalues for each $3\times3$ block and then applying an independent random rotation to each block. 
Stochasticity is introduced by defining a design matrix $X = H^{1/2}$ and, at each iteration, subsampling a random subset of rows of $X$ to form a stochastic approximation of the loss and its gradients.

\subsection{Heavy-Tailed Gradient Noise Benchmark Setup} \label{appx:ahn_linear_benchmark}
We adopt the controlled linear-transformer benchmark introduced by \citet{ahn2023linear}.
In this benchmark, each input sequence corresponds to a distinct linear regression task.
Specifically, an input takes the form $z \triangleq \Big( \begin{pmatrix} \vx_1 \\ y_1 \end{pmatrix}, \begin{pmatrix} \vx_2 \\ y_2 \end{pmatrix}, \cdots, \begin{pmatrix} \vx_n \\ y_n \end{pmatrix}, \begin{pmatrix} \vx_{n+1} \\ 0 \end{pmatrix} \Big),$ where $\vw^\top \vx_i = y_i$ for $i=1,\dots,n+1$ and the latent regression vector $\vw \sim \mathcal N(0,I_d)$ is independently sampled for each sequence.
The learning objective is to predict $y_{n+1}$ from the context $\{(\vx_i,y_i)\}_{i=1}^n$, and training minimizes the mean squared prediction error.
Following \citet{ahn2023linear}, we use dimension $d=5$ and context length $n=20$. This setting has been extensively used to study in-context learning and optimization dynamics of transformers in a simplified yet faithful regime \citep{akyurek2022learning,von2023transformers,garg2022can}.

To study the impact of heavy-tailed stochastic gradients, we consider two covariate distributions. In the \emph{light-tailed} setting, covariates are sampled as $\vx_i \sim \mathcal N(0,I_d)$. In the \emph{heavy-tailed} setting, we sample $\vx_i$ uniformly from the unit sphere $\mathbb S^{d-1}$ and scale each sample by an independent heavy-tailed random variable $\sqrt{\Gamma_{0.1,10}}$, where $\Gamma_{k,\theta}$ denotes the Gamma distribution with shape $k$ and scale $\theta$. 
This construction significantly amplifies tail behavior in the covariates, thereby inducing heavy-tailed gradient noise during optimization and enabling a controlled evaluation of optimizer robustness under extreme stochastic fluctuations.

\section{Ablation Studies} \label{appx:ablation_studies}
We conducted ablation studies on a 130M-parameter Llama model trained on the C4 dataset with RMSProp+Magma, following the setup in \S~\ref{subsec:pre-training-llms}.

\subsection{Masking Component}
We investigated the efficacy of Magma applied to specific transformer sub-modules (Attention vs. MLP) compared to a global masking strategy. As shown in Table~\ref{table:masking-component}, applying masking exclusively to attention blocks improves model performance, reducing validation perplexity from a baseline of 22.64 to 21.92. We observe a synergistic effect when masking is applied to both attention and MLP simultaneously; this configuration achieves the lowest perplexity of 21.65, surpassing the most comprehensive setting (21.94). These findings indicate that targeted regularization of specific sub-modules yields superior performance compared to uniform masking strategies.

\begin{table} 
\caption{Validation perplexity ($\downarrow$) for different masking components.}
\begin{center}
% \resizebox{\textwidth}{!}{%
\begin{tabular}{lccccc}
\toprule
& Baseline & Attn Only & Attn + MLP & All \\
\midrule
Eval Perplexity & 22.64 & 21.92 & 21.65 & 21.94 \\
\bottomrule
\end{tabular}%
% }
\end{center}
\label{table:masking-component}
\end{table}

\begin{table}
\caption{Validation perplexity ($\downarrow$) for different masking granularities using Uniform vs. Momentum-Gradient Alignment-based sampling schemes (with and without damping). The baseline (RMSProp) excluding both sampling and damping achieves the perplexity of 22.64.}
\begin{center}
% \resizebox{\textwidth}{!}{%
\begin{tabular}{lrrrr}
\toprule
Method & Element & Row & Column & Block \\
\midrule  
Uniform Sampling  & 21.73 & 21.76 & 21.78 & 21.81 \\
Damping & 21.97 & 21.95 & 21.91 & 21.92 \\
Uniform Sampling + Damping  & 21.58 & 21.62 & 21.61 & 21.65 \\
Momentum-Gradient Alignment-based Sampling & 21.77 & 21.78 & 21.75 & 21.78 \\
Momentum-Gradient Alignment-based Sampling + Damping & 21.63  & 21.60 & 21.61 & 21.65 \\
\bottomrule
\end{tabular}%
% }
\end{center}
\label{table:masking-granularity}
\end{table}

\subsection{Masking Granularity}
Table \ref{table:masking-granularity} summarizes the validation perplexity across four masking granularities (Element, Row, Column, and Block) using different sampling and damping configurations.

\paragraph{Robustness Across Granularities} We found that changing the masking granularity has very little impact on stability. For example, with Uniform Sampling, the score varies only slightly—from 21.73 (Element) to 21.81 (Block). This indicates that while fine-grained masking provides a small edge, block masking is preferable for its efficiency—it saves significant memory for cosine similarities with minimal loss in accuracy.

\paragraph{Effectiveness of Damping and Sampling} The results in Table \ref{table:masking-granularity} highlight the synergistic effect of combining sampling schemes with damping. While damping alone yields a perplexity improvement over the RMSProp baseline from 22.64 to around 21.92, it does not outperform standalone uniform sampling. However, integrating damping with sampling strategies consistently unlocks the lowest perplexity scores. Uniform sampling + damping achieves the minimum perplexity of 21.58 (Element-level masking), improving upon the standalone uniform sampling by approximately 0.15. On the other hand, momentum-gradient alignment-based sampling, which utilizes the masking probability, does not outperform uniform sampling.

\begin{figure}
    \centering
    \begin{minipage}[b]{0.45\textwidth}
        \centering
        \includegraphics[width=\linewidth]{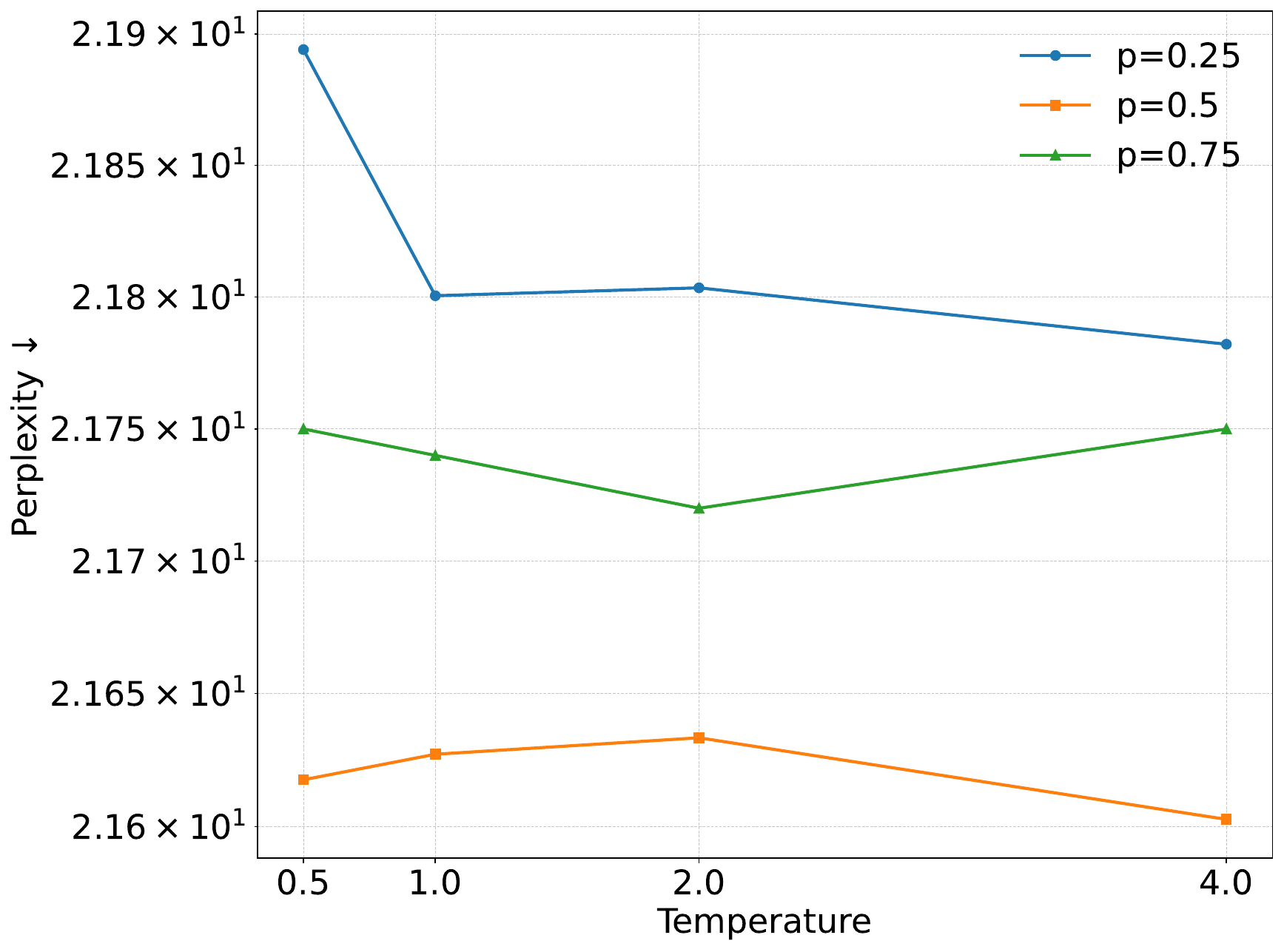} 
        \caption{Comparison of eval perplexity for different values of sampling ratio $p$ and damping temperature $\tau$.}
        \label{fig:sampling-ratio-damping-temperature-sweep}
    \end{minipage}
    \hfill 
    \begin{minipage}[b]{0.45\textwidth}
        \centering
        \includegraphics[width=\linewidth]{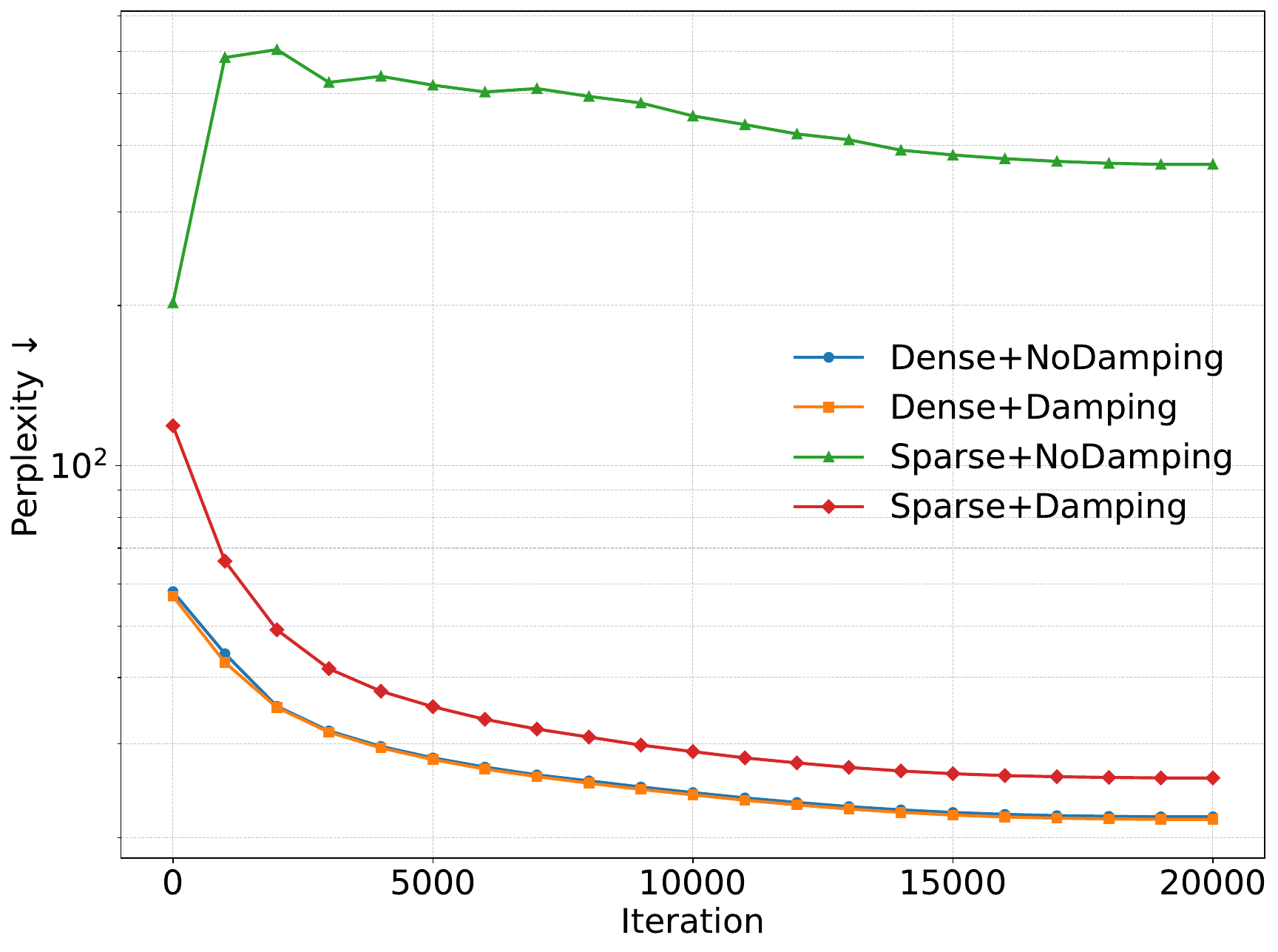} 
        \caption{Comparison of training perplexity on the C4 dataset over 20,000 iterations for Dense and Sparse momentum updates, with and without damping. }
        \label{fig:dense-vs-sparse-momentum-update}
    \end{minipage}
\end{figure}

% \begin{figure}
% % \vskip -0.1in
% \begin{center}
% \includegraphics[width=0.5\columnwidth]{figs/dense_vs_sparse_momentum_update.pdf}
% \end{center}
% \caption{
% Comparison of training perplexity on the C4 dataset over 20,000 iterations for Dense and Sparse momentum updates, with and without damping. 
% % See Appendix \ref{subsec:sparse-vs-dense-momentum-update} for further details. 
% }
% \label{fig:dense-vs-sparse-momentum-update}
% % \vskip -0.3in
% \end{figure}

\subsection{Sampling Ratio and Damping Temperature}
To identify the optimal sampling ratio $p$ and damping temperature $\tau$, we experimented with $p \in \{0.25, 0.5, 0.75\}$ and $\tau \in \{ 0.5, 1.0, 2.0, 4.0\}$. As Figure~\ref{fig:sampling-ratio-damping-temperature-sweep} demonstrates, $p=0.5$ outperforms other values across all temperatures. On the other hand, the results are not very sensitive to temperature, so we proceed with $\tau=2.0$ in all our experiments.

% \begin{figure} 
% % \vskip -0.1in
% \begin{center}
% \includegraphics[width=0.5\columnwidth]{figs/sampling_ratio_and_damping_temperature.pdf}
% \end{center}
% \caption{
% Comparison of eval perplexity for different values of sampling ratio $p$ and damping temperature $\tau$.
% }
% \label{fig:sampling-ratio-damping-temperature-sweep}
% \end{figure}

\begin{figure} 
% \vskip -0.1in
\begin{center}
\includegraphics[width=0.5\columnwidth]{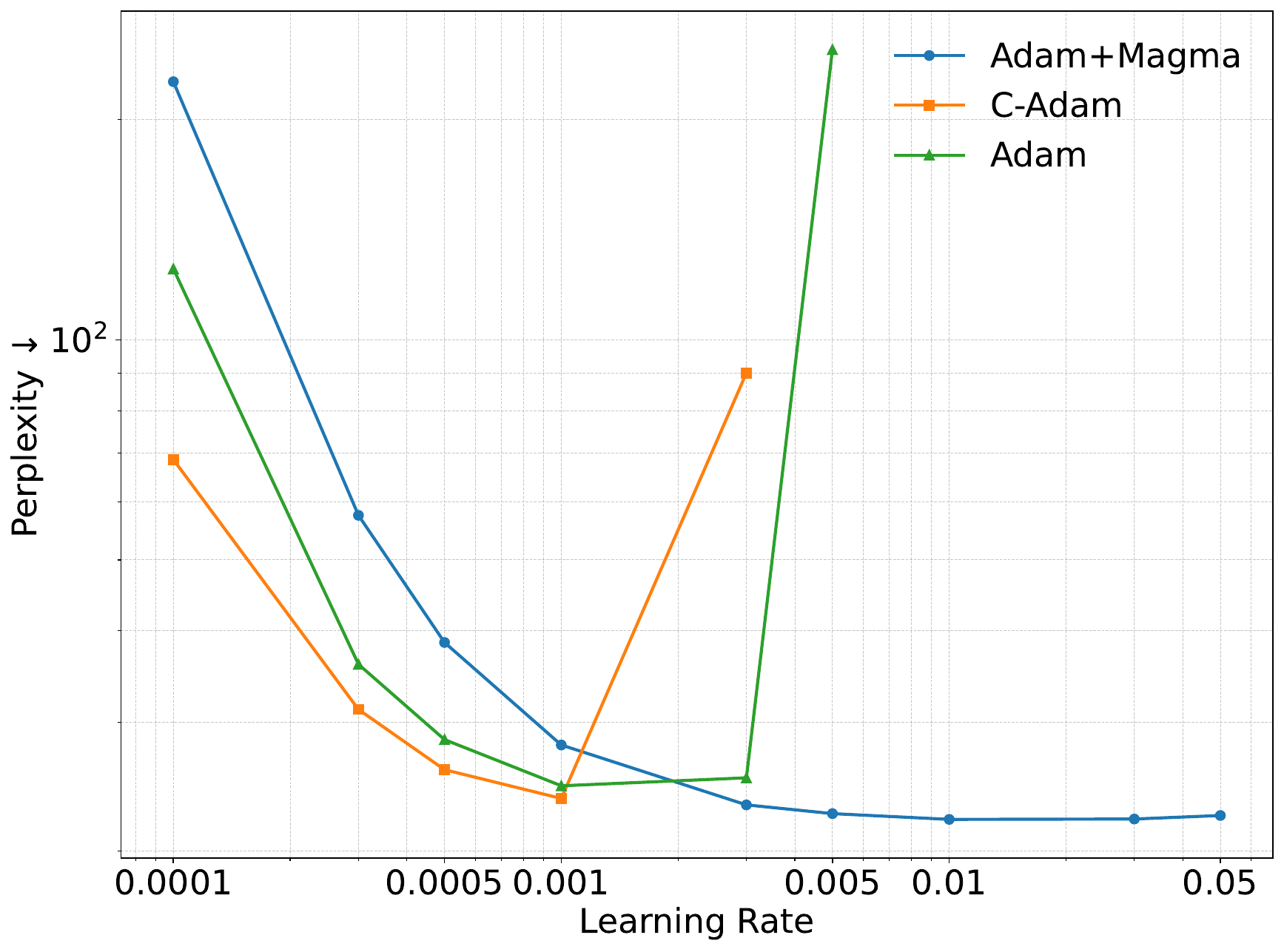}
\end{center}
\caption{Sensitivity analysis of learning rate on evaluation perplexity. Unlike Adam and C-Adam, which exhibit narrow optimal windows, Adam+Magma maintains consistent performance, demonstrating stability across a significantly wider hyperparameter range.}
\label{fig:learning-rate-sensitivity}
% \vskip -0.2in
\end{figure}

\subsection{Sparse vs. Dense Momentum Update}
\label{subsec:sparse-vs-dense-momentum-update}
We consider four different settings with a learning rate of 0.001. The results indicate a critical disparity between dense and sparse update methods. As shown in Figure~\ref{fig:dense-vs-sparse-momentum-update}, the dense baselines—regardless of damping—consistently maintain robust convergence and achieve the lowest perplexity. In contrast, the sparse momentum update without damping exhibits severe instability, characterized by a sharp increase in perplexity that remains high throughout the 20,000 iterations. Although the introduction of damping stabilizes the model, its trajectory still underperforms dense update baselines.

\subsection{Sensitivity to Learning Rate} As Figure~\ref{fig:learning-rate-sensitivity} demonstrates, Adam+Magma exhibits superior robustness to learning rate variations compared to baseline optimizers. While C-Adam and Adam are sensitive to the learning rate—with perplexity spiking when the learning rate deviates from approximately 0.001–0.003—Adam+Magma maintains stability across a broader spectrum. In particular, it remains effective at rates up to 0.05, a region where other optimizers fail to converge. This suggests Adam+Magma reduces the need for precise hyperparameter tuning, offering greater reliability for resource-constrained experimental setups.

\end{document}